%% file: main.tex
\documentclass[conference]{IEEEtran}
\usepackage{times}

\usepackage[numbers,sort&compress]{natbib}

\usepackage{multicol}
\usepackage[bookmarks=true]{hyperref}
\usepackage{graphicx}
\usepackage{xcolor}

\usepackage{physics} 
\usepackage{amsmath}
\usepackage{amssymb}
\usepackage{amsthm}
\usepackage{mathtools}
\usepackage{derivative}
\usepackage{esint}
\usepackage{xcolor}
\usepackage{bm}
\usepackage{mathrsfs}
\usepackage{graphicx}
\usepackage[export]{adjustbox} 
\graphicspath{{figures/}}
\usepackage{stfloats}

\newtheorem{theorem}{Theorem}
\newtheorem{lemma}{Lemma}

\newtheorem{proposition}{Proposition}  
\theoremstyle{definition}
\newtheorem{definition}{Definition}

\newtheorem{remark}{Remark}

\usepackage{caption}
\DeclareMathOperator*{\argmin}{arg\,min}

\begin{document}

\title{Dynamic Safety in Complex Environments: 
\\ Synthesizing Safety Filters with Poisson's Equation}

\author{Gilbert Bahati, Ryan M. Bena, and Aaron D. Ames.

\authorblockA{\\
Department of Mechanical and Civil Engineering,\\ California Institute of Technology, \\ Pasadena, CA, USA.\\
\small{E-mails:\,$\{${gbahati}, {ryanbena}, {ames}$\}${@caltech.edu}}
}
}

\maketitle

\input{def}

\begin{abstract}Synthesizing safe sets for robotic systems operating in complex and dynamically changing environments is a challenging problem. Solving this problem can enable the construction of safety filters that guarantee safe control actions---most notably by employing Control Barrier Functions (CBFs). This paper presents an algorithm for generating safe sets from perception data by leveraging elliptic partial differential equations, specifically Poisson's equation. 
Given a local occupancy map, we solve Poisson's equation subject to Dirichlet boundary conditions, with a novel forcing function. Specifically, we design a smooth \textit{guidance} vector field, which encodes gradient information required for safety.
The result is a variational problem for which the unique minimizer---a \textit{safety function}---characterizes the safe set. After establishing our theoretical result, we illustrate how safety functions can be used in CBF-based safety filtering.
The real-time utility of our synthesis method is highlighted through hardware demonstrations on quadruped and humanoid robots navigating dynamically changing obstacle-filled environments.
\end{abstract}

\IEEEpeerreviewmaketitle

\input{intro}
\noindent
\textit{\textbf{Notation:}}
\begin{itemize}
    \item A continuous function $\gamma: \mathbb{R} \rightarrow \mathbb{R}$ is an \textit{extended class} $\mathcal{K}$, $\Kc^{e}_{\infty}$, $(\gamma \in \Kc^{e}_{\infty})$ if $\gamma$ is monotonically increasing, $\gamma(0) = 0$, $\lim_{s \rightarrow \infty} \gamma(s) = \infty$, and $\lim_{s \rightarrow -\infty} \gamma(s) = -\infty$.
    \item For a function $w:\re^n \rightarrow \re$ where $\bx \mapsto w(\bx) = w(x_1,\cdots,x_n)$, let $Dw$ denote the gradient and $D^2w$ denote the Hessian. More generally, for $k \in \mathbb{N}_0$, $D^kw$ is the collection of all partial derivatives of order $k$. Given a multi-index $\xi = (\xi_1, \cdots,\xi_n) \in \mathbb{N}_0^n$ with $|\xi| = k$:
    \[
    D^k w \coloneqq \big \{D^{\xi}w \, \big | \,  |\xi| = k \big \}, \quad  D^\xi w = \frac{\partial^{|\xi|} w}{\partial x_1^{\xi_1} \cdots \partial x_n^{\xi_n}}. 
    \]
    \item $\Oc$ is an open, bounded and connected set with smooth boundary $\pOc$ such that $\overline{\Oc}  = \Oc \cup \pOc$ is the closure of $\Oc$.
    \item $C^k(\Oc)$ is the set of functions: 
    $ \{w:\Oc\rightarrow \re \, \big | \,$ $w$ is $k$-times continuously differentiable$\}.$  
    \item $C^k(\Occ) 
     = \{w \in  C^k(\Oc) \, | \,$ $D^{\xi}w$ is uniformly continuous on bounded subsets of $\Oc$ for all $|\xi| \leq k$\}, that is,  $D^{\xi}w$ continuously extends to $\Occ$. 
       \item $C^k(\Oc;\re_{\geq 0})$ =
     $ \{w:\Oc\rightarrow \re_{\geq 0} \, \big | \,$ $w$ is $k$-times continuously differentiable$\}$ with the similar respective definitions for $C^k(\Oc;\re_{> 0}), C^k(\Oc;\re_{\leq 0}), C^k(\Oc;\re_{< 0})$ and vector or matrix-valued functions $C^k(\Oc;\re^n),C^k(\Oc;\re^{n \times n}).$
    \item $C^{k,\alpha}(\Oc)$ for $0< \alpha < 1$ denotes Hölder continuous function spaces: see Appendix~\ref{appdx: poisson's equation}.
\end{itemize}

\input{background_cbfs}

\input{our_method}
\input{forcing_functions}
\input{safety_with_poisson}

\input{demonstrations}

\input{conclusion}

\bibliographystyle{ieeetr}
\bibliography{mainbib, main-GB, cohen}

\appendices
\input{appendix_poisson}
\end{document}

%% file: def.tex
\newcommand{\naturals}{\mathbb{N}}
\newcommand{\re}{\mathbb{R}}
\newcommand{\R}{\mathbb{R}}
\newcommand{\realnonneg}{\mathbb{R}_{\ge 0}}
\newcommand{\realpos}{\mathbb{R}_{> 0}}
\newcommand{\until}[1]{[#1]}
\newcommand{\map}[3]{#1:#2 \rightarrow #3}
\newcommand{\qedA}{~\hfill \ensuremath{\square}}
\newcommand\scalemath[2]{\scalebox{#1}{\mbox{\ensuremath{\displaystyle #2}}}}
\newcommand{\interior}{\operatorname{int}}

\newcommand{\longthmtitle}[1]{\mbox{}{\textit{(#1):}}}
\newcommand{\setdef}[2]{\{#1 \; | \; #2\}}
\newcommand{\setdefb}[2]{\big\{#1 \; | \; #2\big\}}
\newcommand{\setdefB}[2]{\Big\{#1 \; | \; #2\Big\}}
\newcommand*{\SetSuchThat}[1][]{} 
\newcommand*{\MvertSets}{%
    \renewcommand*\SetSuchThat[1][]{%
        \mathclose{}%
        \nonscript\;##1\vert\penalty\relpenalty\nonscript\;%
        \mathopen{}%
    }%
}
\MvertSets 

\newcommand{\dt}{\mathrm{d}t}
\newcommand{\dy}{\mathrm{d}y}
\newcommand{\dx}{\mathrm{d}x}
\newcommand{\dtau}{\mathrm{d}\tau}
\newcommand{\Cc}{\mathcal{C}}
\newcommand{\Ac}{\mathcal{A}}
\newcommand{\pCc}{\partial \mathcal{C}}
\newcommand{\Bc}{\mathcal{B}}
\newcommand{\Tc}{\mathcal{T}}
\newcommand{\Dc}{\mathcal{D}}
\newcommand{\Oc}{\Omega}
\newcommand{\Occ}{\overline{\Omega}}
\newcommand{\pOc}{\partial \Omega}
\newcommand{\Ocext}{\Oc_\mathrm{ext}}
\newcommand{\Ocint}{\Oc_\mathrm{int}}
\newcommand{\Hc}{\mathcal{H}}
\newcommand{\Kc}{\mathcal{K}}
\newcommand{\Pc}{\mathcal{P}}
\newcommand{\Uc}{\mathcal{U}}
\newcommand{\Sc}{\mathcal{S}}
\newcommand{\Xc}{\mathcal{X}}
\newcommand{\Yc}{\mathcal{Y}}
\newcommand{\Vc}{\mathcal{V}}
\newcommand{\Zc}{\mathcal{Z}}
\newcommand{\Ec}{\mathcal{E}}
\newcommand{\Rm}{\mathcal{\mathbb{R}}}

\newcommand{\divv}{\nabla \cdot \vec{\bv}}
\newcommand{\hs}{h_\mathrm{\Sc}}

\newcommand{\defeq}{\triangleq}

\newcommand{\vr}{\varepsilon}
\newcommand{\nom}{{\operatorname{nom}}}
\newcommand{\m}{{\operatorname{min}}}
\newcommand{\des}{{\operatorname{des}}}
\newcommand{\on}{{\operatorname{on}}}
\newcommand{\off}{{\operatorname{off}}}
\newcommand{\fl}{{\operatorname{FL}}}
\newcommand{\Lie}{\mathcal{L}}
\newcommand{\qp}{{\operatorname{QP}}}

\newcommand{\ie}{i.e., }
\newcommand{\todo}[1]{{\color{cyan} Todo: #1}}

\newcommand{\ba}{\mathbf{a}}
\newcommand{\bb}{\mathbf{b}}
\newcommand{\be}{\mathbf{e}}
\renewcommand{\bf}{\mathbf{f}} 
\newcommand{\bff}{\mathbf{f}}
\newcommand{\bg}{\mathbf{g}}
\newcommand{\bk}{\mathbf{k}}
\newcommand{\bp}{\mathbf{p}}
\newcommand{\bq}{\mathbf{q}}
\newcommand{\bu}{\mathbf{u}}
\newcommand{\bv}{\mathbf{v}}
\newcommand{\bvv}{\vec{\mathbf{v}}}
\newcommand{\bn}{\mathbf{n}}
\newcommand{\hbn}{\hat{\mathbf{n}}}

\newcommand{\bx}{\mathbf{x}}
\newcommand{\bz}{\mathbf{z}}
\newcommand{\br}{\mathbf{r}}
\newcommand{\bA}{\mathbf{A}}
\newcommand{\bB}{\mathbf{B}}
\newcommand{\bD}{\mathbf{D}}
\newcommand{\bC}{\mathbf{C}}
\newcommand{\bF}{\mathbf{F}}
\newcommand{\bJ}{\mathbf{J}}
\newcommand{\bG}{\mathbf{G}}
\newcommand{\bK}{\mathbf{K}}
\newcommand{\bP}{\mathbf{P}}
\newcommand{\bW}{\mathbf{W}}
\newcommand{\bw}{\mathbf{w}}
\newcommand{\bd}{\mathbf{d}}
\newcommand{\bvy}{\vec{\by}}
\newcommand{\bty}{\tilde{\by}}
\newcommand{\bbeta}{\boldsymbol{\eta}}
\newcommand{\mb}[1]{\mathbf{#1}}

\newcommand{\bY}{\mathbf{Y}}
\newcommand{\by}{\mathbf{y}}
\newcommand{\byobs}{\mathbf{y}_\mathrm{obs}}
\newcommand{\bl}{\mathbf{\lambda}}

\newcommand{\bxd}{\bx_\mathrm{d}}
\newcommand{\bxobs}{\bx_\mathrm{obs}}
\newcommand{\md}{\mathrm{d}}

\newcommand{\Uxd}{U_{\mathrm{d}}}
\newcommand{\Uobs}{U_{\mathrm{obs}}}
\newcommand{\Uapf}{U_{\mathrm{APF}}}

\newcommand{\GradUxd}{\nabla U_{\mathrm{d}}}
\newcommand{\GradUobs}{\nabla U_{\mathrm{obs}}}
\newcommand{\GradUapf}{\nabla U_{\mathrm{APF}}}


%% file: intro.tex
\section{Introduction}

Deploying robotic systems in real-world environments autonomously requires that they operate in complex, dynamic environments while avoiding collisions with multiple objects of arbitrary geometry.  Achieving this level of dynamic safety necessitates a quantifiable description of the safety requirement, i.e. a functional representation of the environment via a safety constraint. Additionally, this representation must be integrated with the dynamics of the system to produce safe inputs, i.e., inputs that ensure satisfaction of the safety requirements.  
To achieve this, numerous methods for enforcing dynamic safety constraints have been investigated, including Hamilton-Jacobi reachability \cite{BansalCDC17,TomlinTAC21,KimAJTsafetyfilters}, Model Predictive Control \cite{borrelli2017predictive,WabersichTAC23}, Artificial Potential Fields (APFs) \cite{khatib1986real,singletary2021comparative}, and Control Barrier Functions (CBFs) \cite{AmesTAC17}.  This paper seeks to fuse the synthesis of safety constraints with controllers that yield safe behaviors, coupling functional representations of dynamic environments with the generation of safe controllers. 

Safety constraints are typically synthesized using heuristic approaches for simple environments, or through general methods that are non-constructive. While heuristic methods work for environments with simple geometries \cite{paulnonsmooth, molnar2023composing}, they struggle when generalized to more complex environments.
Alternatively, for more complex geometries, signed distance functions (SDFs) \cite{SingletaryMolnarSafetyCriticalFood,oleynikova2017voxblox,long2021learning} have been proposed as a method for generating safety descriptors in the context of collision avoidance.
However, SDFs possess gradient discontinuities which present challenges in synthesizing safe controllers.
Therefore, the synthesis of safety constraints is a challenging problem in its own right---made even more challenging by the need to couple safety constraints with system dynamics. 

\begin{figure}[t!]
    \centering    \includegraphics[width=1\linewidth]{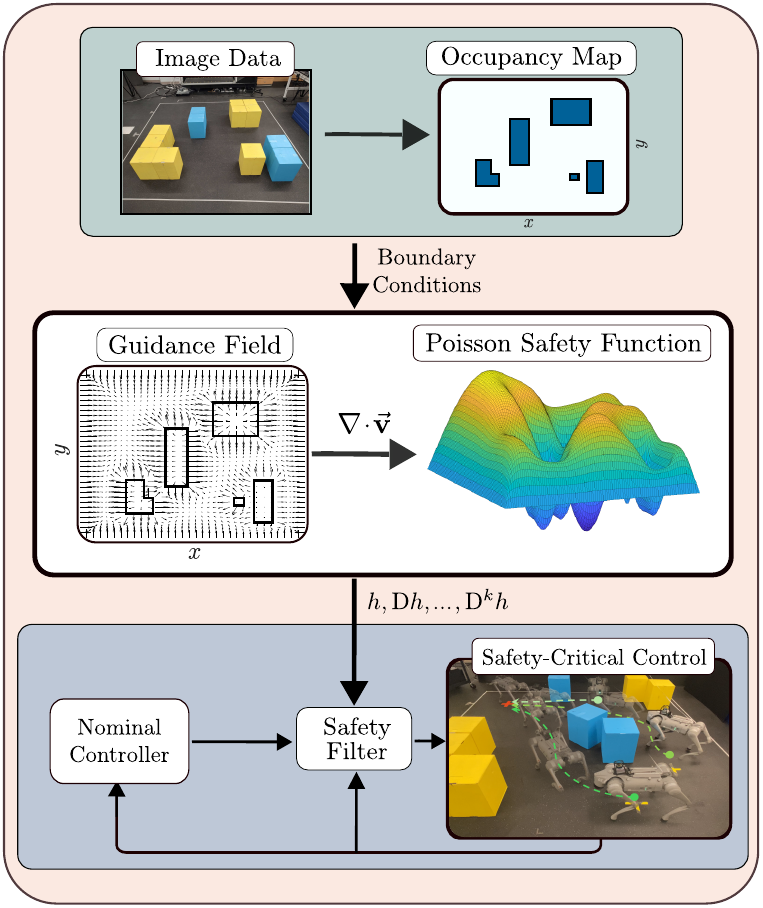}
    \vspace{-5mm}
    \caption{Safe set synthesis from perception data via Poisson's equation. Hardware experimental footage: {\color{blue}\url{https://youtu.be/fBRdkAJGixI}.}}
    \label{fig:hero_figure}
    \vspace{-8mm}
\end{figure}

Given a description of the environment in the context of a goal location subject to obstacles, a popular method for enforcing safe behavior is the APF.  Originally introduced in \cite{khatib1986real}, APFs are a heuristic technique for achieving dynamic safety by encoding ``repulsive" gradients on obstacle surfaces and ``attractive" gradients to a desired goal. This coupling between the dynamics of the system (via gradients) and specifications (via attractive and repulsive potentials) allows for the problem to be framed as one of constructing an APF. 
One approach to constructing APFs is to frame them as solutions to elliptic partial differential equations (PDEs) \cite{gilbarg1977elliptic,jost2012partial}. This has been employed in the context of navigation and path planning \cite{connolly1990path,connolly1993applications,kim1992real} and is useful for complex geometry \cite{Davidfluidpotential}.
Despite the success of APFs, this method assumes single integrator dynamics, i.e., a kinematic model.  Additionally, APFs couple goal-reaching (stability) and collision avoidance (safety), limiting their ability to provide a description of safety that is independent of control objectives.   

A decoupled approach to synthesizing safe controllers for general nonlinear systems is through CBFs \cite{ADA-JWG-PT:14,AmesTAC17}.  In this framework, safety constraints yield a safe set defined as the superlevel set of a function, and this function is a CBF if there exists inputs that render the safe set forward invariant. CBFs have been shown to generalize APFs \cite{singletary2021comparative}, in that they decouple safety from stability, and apply to general nonlinear systems. Specifically, given a CBF, one can synthesize a safety filter \cite{gurriet2018towards,AmesECC19} framed as a quadratic program with the stability objectives as a cost and safety objective as a constraint.  
Methods such as \cite{TamasRAL22, cohen2024constructive,bahati2025control} offer a systematic approach for constructing CBF-based safety filters for robotic systems via reduced-order models with real-time realizations on hardware. 
However, for collision avoidance in complex environments, the constructive design of functions encoding safety objectives, i.e., the synthesis of safe sets and the corresponding CBFs, remains a challenging task. 

In this paper, we show that elliptic PDEs can be used to synthesize \emph{safety functions}---functional representations of complex environments that characterize safety. We demonstrate that the process of constructing safe sets can be formulated as solving a Dirichlet problem for Poisson's equation. 
In particular, given an occupancy map, we solve Poisson's equation subject to Dirichlet boundary conditions, where the solution, a safety function, characterizes a safe set, while the boundary conditions encode desired level-set values on obstacle surfaces.
We propose several methods for constructing the forcing function within Poisson's equation, including an average flux method and a guidance field method \cite{perez2023poisson} that provides additional degrees of freedom for defining safety.  The solution has desirable differentiability (\ie regularity) properties, which are critical in enabling their use in the synthesis of safety filters. 

The key observation of this paper is that safety functions obtained from Poisson's equation can be used to synthesize CBFs and, therefore, safety filters.  The main contributions are threefold: (1) we present a constructive way of generating safe sets for complex environments from perception data via Poisson's equation, (2) we illustrate and prove how the resulting safety functions can be used to synthesize CBF-based safety filters, and (3) we demonstrate the real-time efficiency of our approach with hardware experiments on quadruped and humanoid robots in static and dynamically changing environments. 
Importantly, we show that safety function synthesis (1) can be done in real-time via perception data to enable safe behaviors (2) on highly dynamic robotic systems (3). 

%% file: background_cbfs.tex
\section{Background:  
Safety-Critical Control}
This section reviews safety for nonlinear systems in the context of CBFs.
We consider a nonlinear control-affine system:
    \begin{align}\label{eq:nl system}
        \dot{\bx} = \bf(\bx) + \bg(\bx)\bu,
    \end{align}
    where $\bx \in \re^n$ is the state and $\bu \in \re^m$ is the control input. The function $\bf:\re^n\rightarrow \re^{n}$ denotes the drift dynamics and $\bg:\re^n\rightarrow \re^{n\times m}$ is the actuation matrix, both assumed to be locally Lipschitz continuous.
    A locally Lipschitz continuous controller $\bk: \re^n \rightarrow \re^m$ yields the closed-loop system:
     \begin{align}\label{eq:ode closed loop}
        \dot{\bx} = \bf_{\mathrm{cl}}(\bx) = \bf(\bx) + \bg(\bx)\bk(\bx).
    \end{align}
    Because the functions $\bf, \bg$, and $\bk$ are locally Lipschitz continuous, \eqref{eq:ode closed loop} defines an ordinary differential equation such that for any initial condition $\bx(0) = \bx_0 \in \re^n$, there exists a unique continuously differentiable solution $t \mapsto \bx(t)$ on the maximal time interval $I_\mathrm{max}(\bx_0) = [0, t_{\max} (\bx_0))$ \cite{perko2013differential}.

\subsection{Control Barrier Functions and Safety Filters}

The concept of safety can be formalized by requiring that all system trajectories $t \mapsto \bx(t)$ remain within a set $\Sc \subset \re^n$, characterized by the notion of forward invariance.
\begin{definition}(Forward Invariance)  A set $\Sc$ is \textit{forward invariant} with respect to \eqref{eq:ode closed loop} if for every initial condition $\bx_0 \in \Sc$, the resulting trajectory $t \mapsto \bx(t) \in \Sc$ for all $t \in I_\mathrm{max}(\bx_0)$. 
\end{definition}
%
We consider a system \textit{safe} with respect to a user-defined \textit{safe set} $\Sc$ if $\Sc$ is forward invariant. In particular, we consider safe sets defined as the 0-superlevel set of a continuously differentiable function $h_\mathrm{\Sc}:\re^n \rightarrow \re$:
\begin{align}\label{eq: safe set og}
\Sc = 
\big \{ 
\bx \in \re^n \,\big|\, \hs(\bx) \geq 0
\big \}.
\end{align}
Forward invariance of such sets can be guaranteed by Nagumo's theorem \cite{blanchini1999set}, which requires the derivative of $\hs$ along the trajectories of the closed-loop system \eqref{eq:ode closed loop} to be non-negative on the boundary of $\Sc$, that is:
\begin{align}\label{eq: Nagumo hdot}
    \dot{h}_\Sc(\bx) = \underbrace{D\hs(\bx)\cdot \bf(\bx)}_{L_\bf \hs(\bx)} +   \underbrace{D\hs(\bx)\cdot\bg(\bx)}_{L_\bg \hs(\bx)} \bk(\bx) \geq 0,
\end{align}
for all $\bx \in \partial \Sc$. This ensures that on the set boundary $\partial \Sc$, the vector field of \eqref{eq:ode closed loop} points towards the interior of $\Sc$ or is tangent to $\partial \Sc$. The condition $D\hs(\bx) \neq \mathbf{0}$ when $\hs(\bx) = 0$ is necessary to ensure the existence of a locally Lipschitz continuous controller $\bk$ for the system \eqref{eq:ode closed loop} that satisfies \eqref{eq: Nagumo hdot}. 
CBFs are a tool for synthesizing controllers and safety filters that enforce the safety of the system \eqref{eq:nl system} on $\Sc$.

\begin{definition}(Control Barrier Function (CBF) \cite{AmesTAC17}) Let $\Sc \subset \re^n$ be the 0-superlevel set of a  continuously differentiable function $\hs:\re^n \rightarrow \re$ satisfying $D\hs(\bx) \neq 0$ when $\hs(\bx) = 0$. The function $\hs$ is a Control Barrier Function (CBF) for \eqref{eq:nl system} on $\Sc$ if there exists $\gamma \in \Kc_{\infty}^e$ such that for all $\bx \in \re^n$:
    \begin{align}\label{eq: CBF main def}
        \! \! \! \! \sup_{\bu \in \re^m}   \!\left \{ {\dot{h}_\Sc(\bx, \bu) \!= \!L_\bf h_\Sc(\bx)} +   {L_\bg h_\Sc(\bx)} \bu > -\gamma(h_\Sc(\bx)) \!\right 
 \} \!.
\! \!   \! \end{align}
\end{definition}
Given a CBF $h_\Sc$ and $\gamma$, the set of feasible point-wise control values satisfying \eqref{eq: CBF main def} is given by:
\begin{align}
    \Kc_\mathrm{CBF}(\bx) = \left \{\bu \in \re^m \, \big | \,  \dot{h}_\Sc(\bx, \bu)  \geq -\gamma(h_\Sc(\bx))  \right \}
\end{align}
such that any locally Lipschitz controller $\bx \mapsto \bk(\bx) \in \Kc_\mathrm{CBF}(\bx)$ enforces the forward invaraince of $\Sc$ \cite{AmesTAC17}, establishing the safety of \eqref{eq:nl system} on $\Sc$.
Given a desired (potentially unsafe) nominal controller $\bk_{\mathrm{nom}}:\re^n \rightarrow \re^m$, the following optimization-based controller \textit{filters} $\bk_{\mathrm{nom}}$ by minimally adjusting it to the nearest safe action:
\begin{align*}\label{eq: safety filter}
    \bk(\bx) = &\argmin_{\bu \in \re^m} &&\|\bu - \bk_{\mathrm{nom}}(\bx)\|_2^2 \tag{Safety-Filter} \\
    & \quad 
 \mathrm{s.t.} &&L_\bf \hs(\bx) + L_\bg \hs(\bx)\bu  \geq - \gamma(\hs(\bx)).
\end{align*}
The next subsection discusses systems with properties that enable a systematic approach to constructing CBFs.
\subsection{Outputs and Relative Degree}
In this paper, we focus on systems for which safety specifications are expressed using a set of desired \textit{outputs}. To facilitate the construction of CBFs, we recall the notion of \textit{relative degree}, which represents the layer of differentiation at which the control inputs affects the outputs.
\begin{definition}[Relative Degree $r$ \cite{isidori1985nonlinear}]\label{def:relative-degree}
    A function $\by\,:\,\R^n\rightarrow\R^p$ has \emph{relative degree} $r\in\mathbb{N}$ for \eqref{eq:nl system} if:
    \begin{align}
        L_{\bg}L_{\bf}^{i}\by(\bx) \equiv \mathbf{0}, &\quad \forall i\in\{0,\dots,r-2\},\label{eq:zero-lie-derivatives} \\
        \mathrm{rank}(L_{\bg}L_{\bf}^{r-1}\by(\bx)) = p, & \quad \forall \bx\in\R^n.\label{eq:decoupling-full-rank}
    \end{align}
\end{definition}
Given an output $\by$ with relative degree $r$, we define a new set of partial coordinates:
\begin{equation}\label{eq:output-coordinates}
      \vec{\mb{y}}(\bx) \coloneqq \begin{bmatrix}
        \by(\bx) \\
        \by^{(1)}(\bx) \\
        \vdots \\
        \by^{(r-1)}(\bx)
    \end{bmatrix}
    =
    \begin{bmatrix}
        \by(\bx) \\
        L_{\bf}\by(\bx) \\
        \vdots \\
        L_{\bf}^{r-1}\by(\bx)
    \end{bmatrix}
    \in\re^{pr},
\end{equation}
where $\by^{(r)} = \dv[r]{\by}{t}$, leading to the following linear dynamics:
\begin{align}\label{eq: linear output dynamcis1}
    \frac{\mathrm{d}}{\dt}\vec{\mb{y}}(\bx) & = 
    \underbrace{
    \begin{bmatrix}
        \mb{0} & \mb{I}_{p(r-1)}\\
        \mb{0} & \mb{0}
    \end{bmatrix}}_{\mb{A}} \vec{\mb{y}}(\bx)
    \! + \!
    \underbrace{
    \begin{bmatrix}
        \mb{0} \\
        \mb{I}_p
    \end{bmatrix}}_{\mb{B}}
    \bw\\
    \bw &\coloneqq L_{\bf}^{r}\by(\bx) + L_{\bg}L_{\bf}^{r-1}\by(\bx)\bu,     \label{eq: y_r mapping}
\end{align}
where \eqref{eq: y_r mapping} is an input to \eqref{eq: linear output dynamcis1}. When $\by$ has relative degree $r$, the controller $\bw= \hat{\bk}(\vec{\by})$ designed for \eqref{eq: linear output dynamcis1} can be transferred back to the controller for \eqref{eq:nl system} as follows\footnote{Condition \eqref{eq:decoupling-full-rank} implies the right psuedo-inverse $L_{\bg}L_{\bf}^{r-1}\by(\bx)^\dagger$ exists.}:
\begin{equation}\label{eq:input-transformation}
    \bu = L_{\bg}L_{\bf}^{r-1}\by(\bx)^\dagger\left[\hat{\bk}(\vec{\by}(\bx)) - L_{\bf}^{r}\by(\bx) \right].
\end{equation}
This partial coordinate transformation is a full coordinate transformation if $pr=n$. The output dynamics \eqref{eq: linear output dynamcis1} are a chain of integrators, allowing for the use of techniques such as \cite{cohen2024constructive,bahati2025control} and \cite{AndrewCDC22,MurrayACC20} to construct CBFs, including for classes of systems with outputs of nonuniform relative degree.

Given an output $\by$, the success of the techniques in \cite{AndrewCDC22,MurrayACC20, cohen2024constructive,bahati2025control} rely on knowledge of the function $\by \mapsto h(\by)$ encoding a safety specification for a desired safe set $\Cc$. 
For simple environments, one can often leverage analytical expressions to characterize the safe set with a smooth function $h$; however, for more complex environments, such as those with multiple obstacles of aribtrary geometry, obtaining a single smooth $h$ is challenging without introducing excess conservatism \cite{molnar2023composing}. In what follows, we demonstrate how Poisson’s equation can be leveraged to overcome these challenges and generate a single smooth function $h$ for environments with arbitrary geometries.

%% file: our_method.tex
\section{Safe Set Synthesis via Poisson's Equation}
In this work, we consider systems with spatial safety specifications, with outputs given by:
\begin{align}\label{eq: spatial coordinates}
    \by \coloneqq \by(\bx) = (x, y, z) \in \re^3.
\end{align}
Our primary focus is the construction of safe sets purely in spatial coordinates \eqref{eq: spatial coordinates}; therefore, our construction is geometric and independent of system dynamics \eqref{eq:nl system}. 
Given occupancy data in the coordinates \eqref{eq: spatial coordinates}, we consider an open domain $\Oc \subset \re^3$ representing unoccupied regions, with its boundary $\partial \Oc$ corresponding to obstacle surfaces. 
Our goal is to design a safe set defined as the $0$-superlevel set of a \textit{safety function}. 
\begin{definition}(Safety Function)\label{def: safety function} Let $\by = (x,y,z) \in \re^3$ represent three-dimensional spatial coordinates. We call a function $h: \Occ \subset \re^3 \rightarrow \re$ a \textit{safety function} of order $k$ on $\Occ$ if $h$ is $k$-times differentiable, $Dh(\by) \neq 0$ when $h(\by) = 0$ and the 0-superlevel set of $h$ defines a safe-set $\Cc$ satisfying:
\begin{subequations}\label{eq: safe set characterizations}
\begin{align}
\Cc = 
\big \{ 
\by \in \Occ \, | \,h(\by) \geq 0
\big \}, \\
\partial \Cc = \big \{ 
\by \in \Occ  \, \big| \, h(\by) = 0
\big \},   \\
\mathrm{int}(\Cc) = \big \{ 
\by \in \Occ  \, \big | \, h(\by) > 0
\big \}.
\end{align}
\end{subequations}
\end{definition}
Specifically, let $\Omega \subset \re^3$ be an open, bounded and  connected set with piecewise smooth boundary $\partial \Oc$. That is,  we assume $\partial \Oc$ to be a finite union of closed, smooth obstacle surfaces:
\begin{align}\label{eq: piecewise boundary}
    \partial \Oc = \bigcup_{i=1}^{n_\mathrm{obs}} \partial \Gamma_i, 
\end{align}
where each $\Gamma_i$ is an open, bounded and connected set defining an obstacle interior, and $\partial \Gamma_i$ is its smooth boundary.  The term $n_\mathrm{obs}$ denotes the total number of obstacles, i.e., isolated occupied regions, in the environment.

We propose a method for synthesizing safe sets from environmental boundary data by solving a boundary value problem for Poisson's equation, a second-order linear elliptic partial differential equation (PDE). 
In particular, we consider safe sets characterized by safety functions which satisfy Poisson's equation subject to Dirichlet boundary conditions:
\begin{gather}\label{eq: poisson's eq}
\left \{
    \begin{aligned}
        \Delta h(\by) &= f(\by)& \text{ in } \Omega,\\
        h(\by) &= 0 &  \text{ on } \partial \Omega, \\
    \end{aligned}
    \right.
\end{gather}
where \( \Delta = \frac{\partial^2 }{\partial x^2} + \frac{\partial^2 }{\partial y^2} + \frac{\partial^2 }{\partial z^2}\) is the \textit{Laplacian} and $f: \Oc \rightarrow \re$ is a given \textit{forcing} function. From classical elliptic regularity stated in Theorem~\ref{thm:regularity} in  Appendix~\ref{appdx: poisson's equation}, a sufficient condition for \eqref{eq: poisson's eq} to admit a twice continuously differentiable solution, $h \in C^2(\Occ)$, is that $f$ is Hölder continuous, that is, $f \in C^{k, \alpha}(\Occ)$ for some $k \in \mathbb{N}_0$ and $0 < \alpha < 1$. Furthermore, a smooth forcing function, $f \in C^\infty(\Occ)$, implies a smooth solution, $h \in C^\infty(\Occ)$, to \eqref{eq: poisson's eq}. To assist the reader, we review key concepts from elliptic PDEs, including Hölder continuity, in Appendix~\ref{appdx: poisson's equation}.

From the weak minimum principle in Theorem~\ref{thm: weak maximum}, a function $h \in C^2(\overline \Omega)$ attains its minimum on $\partial \Omega$ if it is superharmonic, \ie  $\Delta h(\by) \leq 0$, in the interior $\Oc$. To ensure $h$ is not a constant function, a sufficient condition is that the strict inequality $\Delta h(\by) < 0$ holds in $\Oc$. Then, from Hopf's lemma, (Lemma~\ref{lemma: Hopf's}), the outward directional derivatives on the boundary satisfy: 
\begin{align}\label{eq: Hopf Lemmainequality}
   D h(\by) \cdot \hat \bn(\by) < 0 \text{ on } \partial \Oc.
\end{align}
A function $h$ satisfying $\Delta h(\by) < 0$ in $\Oc$ can be obtained by solving \eqref{eq: poisson's eq} with a negative forcing function, \ie $f(\by) < 0$ for all $\by \in \Oc$. In particular, if $f: \Oc \rightarrow \re_{<0}$ is Hölder continuous on $\Oc$, we obtain the following theorem.

\begin{theorem}(Poisson Safety Function)\label{thm: main thm Safety Value Function}
    Let $\Oc$ be an open, bounded and connected set with piecewise smooth boundary $\pOc$. Suppose $f \in C^{k,\alpha}(\Oc;\re_{< 0})$ for some $k \in \mathbb{N}_0$ and $ \alpha \in (0,1)$. 
    Then the solution $h:\Occ \rightarrow \re$ to the Dirichlet problem for Poisson's equation \eqref{eq: poisson's eq} is a safety function of order $2+k$.
\end{theorem}
\begin{proof}
From the minimum principle and classical elliptic regularity presented in Theorem \ref{thm: weak maximum} and \ref{thm:regularity} respectively, we have that $f \in C^{k,\alpha}(\Oc;\re_{< 0})$ implies that $h \in C^{2+k,\alpha}(\overline{\Omega};\re_{\geq 0})$, which defines a compact 0-superlevel set satisfying \eqref{eq: safe set characterizations} with $\Cc = \overline{\Oc},
\partial \Cc = \partial \Oc$ and 
$\mathrm{int}(\Cc) = \Omega
$. Given that the boundary is piecewise smooth as in \eqref{eq: piecewise boundary}, the interior sphere property \ie Def. \ref{def: interior sphere property} in Appendix~\ref{appdx: poisson's equation} holds. Thus from Hopf's lemma, we have that \eqref{eq: Hopf Lemmainequality} holds, implying $D h(\by) \neq 0$ when $h(\by)=0$. Thus, $h$ is a safety function of order $2 +k$ as in Def.~\ref{def: safety function}. 
\end{proof}

The above theorem provides conditions to ensure a solution $h$ to \eqref{eq: poisson's eq} characterizes safe regions.
Unsafe regions are typically defined by regions where $h(\by) < 0$. 
By letting $\Gamma_i$ (\ie occupied region corresponding to the interior of each obstacle) define an unsafe region, 
one can verify from the weak maximum principle in Theorem~\ref{thm: weak maximum}, that solving Poisson's equation \eqref{eq: poisson's eq} on $\overline{\Gamma}_i$ with $f \in C^{k,\alpha}(\Gamma_i;\re_{> 0})$ ensures the solution $h$ is subharmonic in each $\Gamma_i$, resulting in $h(\by) < 0$ for all $\by \in \Gamma_i$. Combining this with Theorem~\ref{thm: main thm Safety Value Function}, the safety function characterizes both safe and unsafe regions.

\begin{remark}(Smooth Boundary)
By assuming a smooth boundary $\pOc$ in \eqref{eq: piecewise boundary}, where each $\partial \Gamma_i$ is smooth for all $i$, any sharp corners on obstacles are assumed to be smoothed out. 
This ensures that the regularity properties (see Appendix~\ref{appdx: poisson's equation}) in the interior of the domains $\Gamma_i$ and $\Omega$ can be extended to the boundaries. 
 Our results still hold for non-smooth boundaries, provided that they satisfy the interior sphere property in Def.~\ref{def: interior sphere property}
with the appropriate (possibly non-classical) definition of the boundary derivative \cite{rosales2019generalizing, apushkinskaya2022around,grisvard2011elliptic, borsuk2006elliptic,krylov1996lectures}.
We present Theorem \ref{thm: main thm Safety Value Function} for smooth boundaries to avoid technical details beyond the scope of this paper, and (as will be discussed later) due to the observed benefits in the performance of synthesized safety-filters near smooth corners.
\end{remark}

In summary, the differentiability (\ie regularity) of $f$ in \eqref{eq: poisson's eq} ensures that the solution $h$ to \eqref{eq: poisson's eq} inherits the desired regularity properties which provide the necessary differentiability guarantees for $h$. Furthermore, the negativity of $f$ in $\Oc$ guarantees that $h$ characterizes the safe set \eqref{eq: safe set characterizations} and satisfies \eqref{eq: Hopf Lemmainequality}, by enforcing the superharmoniticity of $h$ in $\Oc$. Thus, $h$ is a safety function on $\Occ$ with guarantees of differentiability. The next section provides methods of constructing forcing functions that satisfy these conditions, facilitating the practical implementation of the resulting safety functions.

%% file: forcing_functions.tex
\section{Forcing function Construction}\label{sec: forcing func construction}
In this section, we present methods of designing forcing functions that ensure the solution to the boundary value problem for Poisson’s equation \eqref{eq: poisson's eq} is a safety function.

\begin{figure*}
    \centering    \includegraphics[width=1\linewidth]{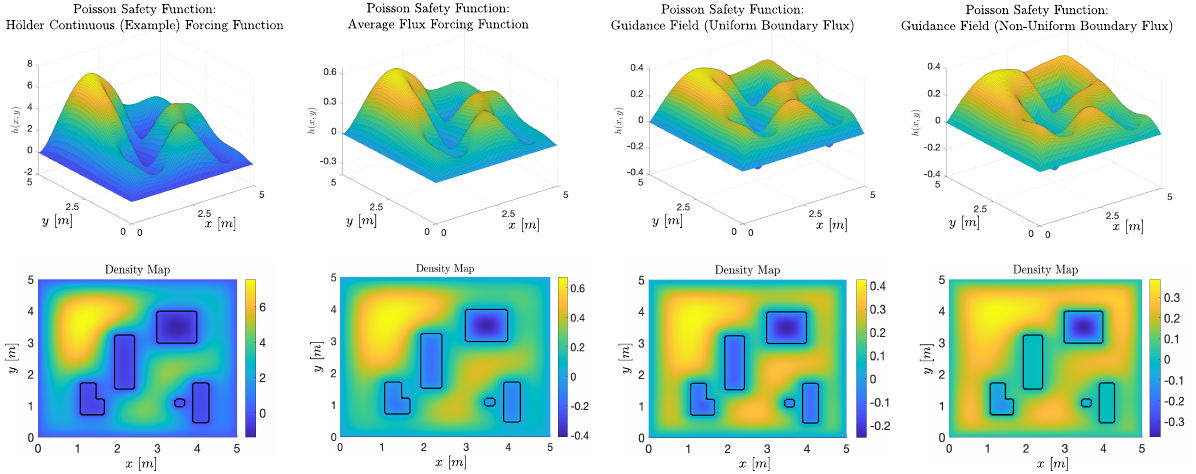}
    \vspace{-6mm}
    \caption{\textbf{[From left to right]} Solutions to Poisson’s equation \eqref{eq: poisson's eq} with the following forcing functions: \textbf{[left]} a Hölder continuous function \eqref{eq: holder continous f} with $\alpha = 0.1$; \textbf{[mid-left]} an average flux forcing function with $\bar b = -1$ in \eqref{eq: forcing function average flux}; \textbf{[mid-right]} a smooth forcing function \eqref{eq: softplus forcing function} constructed using the guidance field \eqref{eq: guidance field} under a uniform boundary flux $b(\by) = -1$ for all $\by \in \pOc$; and \textbf{[right]} the same forcing function with a non-uniform boundary flux $b: \pOc \to \mathbb{R}_{<0}$, allowing different flux values across regions of the boundary, corresponding to different obstacles.}
    \label{fig: forcing funcs}
    \vspace{-4mm}
\end{figure*} 

\subsection{Direct Assignment}
\subsubsection{\textbf{Distance Metric}} An approach to constructing a Hölder continuous forcing function $f \in C^{0,\alpha}(\Oc;\re_{<0})$ for $\alpha \in (0,1)$ is based on the distance to obstacles:
\begin{align}
    \mathrm{dist}(\by, \partial \Oc) = \min_{\byobs \in \partial \Oc} \|\by - \byobs\|,
\end{align}
\begin{align}\label{eq: holder continous f}
    f(\by) = -\left(\frac{\mathrm{dist}(\by, \partial \Oc)}{\|\mathrm{dist}(\by, \partial \Oc)\|_{\infty}}\right)^\alpha, 
\end{align}
for all $\by \in \Occ$. One can verify that \eqref{eq: holder continous f} is Hölder continuous on $\Oc$ by leveraging  \cite[Proposition 1.1.2]{fiorenza2017holder}.
%
Following from Theorem~\ref{thm: main thm Safety Value Function}, the forcing function \eqref{eq: holder continous f} yields a safety function $h \in C^{2, \alpha}(\overline{\Oc}; \re_{\geq 0})$ that lacks orders of differentiability higher than 2. This limitation makes this choice of $f$ unsuitable for control design for systems with outputs of relative degree $r > 1$ as defined in Def.~\ref{def:relative-degree}.
\\
\subsubsection{\textbf{Constant Value}} 
An alternative, straightforward choice of a forcing function is a constant negative value $f \in \re_{<0}$, which guarantees that $h\in C^\infty(\overline{\Omega}; \re_{\geq 0})$.
However, this approach indirectly assigns arbitrary flux values (encoding the magnitude of ``repulsive" gradients) on $\partial \Oc$ without precise knowledge of the resulting flux magnitude.  The notion of flux is detailed in Appendix~\ref{appdx: poisson's equation}.
While this approach yields the desired safe set properties \eqref{eq: safe set characterizations}, it may be beneficial to relate the flux on $\pOc$ to geometric properties of the domain. 
One way to achieve this is by designing a forcing function that enforces a desired \textit{average} flux on $\pOc$. To do this, we leverage Gauss's divergence theorem presented in Appendix~\ref{appdx: poisson's equation}, which dictates that all forcing functions must satisfy\footnote{Note $\pOc$ is the $0$-level set, so $Dh$ is normal (\ie perpendicular) to $\pOc$.}:
\begin{align}
\iiint_\Omega f(\by) \,  \mathrm{d} V 
&= \iiint_\Omega \Delta h(\by) \, \mathrm{d} V \\
&= \oiint_{\partial \Omega} D h(\by) \cdot \hat{\mathbf{n}}(\by) \, \mathrm{d} A,
\end{align} 
where $\mathrm{d} V \coloneq \dx\dy \mathrm{d}z$ denotes the volume element of $\Oc$ and $\mathrm{d} A$ is the surface element of $\partial \Oc$. Let $\Bar{b} \in \re_{<0}$ denote the desired average flux on $\partial \Oc$, defined as:
\begin{equation}
\Bar{b} \coloneqq \frac{
1
}{ \mathrm{Area}(\partial \Oc) }\oiint_{\partial \Omega} D h(\by) \cdot \hat{\mathbf{n}}(\by) \, \mathrm{d} A,
\end{equation}
where $\mathrm{Area}(\partial \Oc)$ denotes the surface area of $\partial \Oc$, representing the surface area of all obstacles. Assuming $f \in \re$, we define:
\begin{equation}\label{eq: forcing function average flux}
f = \frac{
\oiint_{\partial \Omega} D h(\by) \cdot \hat{\mathbf{n}}(\by) \, \mathrm{d} A
}{\iiint_\Omega  \mathrm{dV} } = \Bar{b} \ \frac{\mathrm{Area}(\partial \Oc)}{\mathrm{Vol}(\Oc)},
\end{equation}
where ${\mathrm{Vol}(\Oc)}$ denotes the volume of the domain $\Omega$. It then follows that the forcing function \eqref{eq: forcing function average flux} is a constant negative value that yields a smooth solution, $h \in C^{\infty}(\Occ; \re_{\geq 0})$, to \eqref{eq: poisson's eq} with the average flux $\bar{b}$ on $\pOc$. 
Fig~\ref{fig: forcing funcs} shows the resulting Poisson safety functions leveraging the above methods.

\begin{figure}[t] 
    \centering    \includegraphics[width=1.\linewidth]{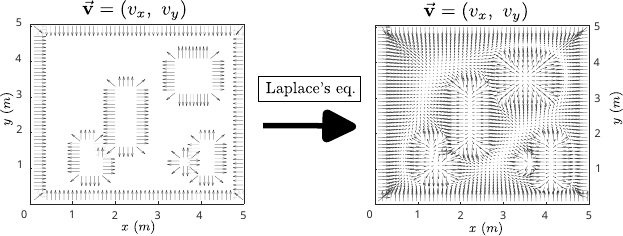}
    \vspace{-6mm}
    \caption{Smooth guidance field generation via Laplace's equation \eqref{eq: guidance field}. \textbf{[left]} Boundary conditions $\vec{\bv} = b \hat{\bn}$ encoding the desired negative flux on obstacle surfaces; and \textbf{[right]} solution to Laplace's equation for each component in $\vec{\bv} = (v_x, v_y)$.} 
    \vspace{-5.5mm}
    \label{fig: guidance field}
\end{figure}

\subsection{Indirect Assignment -- Variational Approach}
Another approach of constructing a forcing function is by designing a \textit{guidance} vector field $\vec \bv:\Occ \rightarrow \re^3$, which encodes the desired flux on $\partial \Oc$, and seeking a function $h$ whose gradient approximates $\vec \bv$ in $\Oc$. Specifically, let $h$ be the minimizer of the cost functional:
\begin{gather}\label{eq: variation cost}
    I[h] =  \iiint_{\Oc} \ \frac{1}{2}|D h(\by)  - \vec{\bv}(\by)|^2  \ \mathrm{d}V,
\end{gather}
among all functions  $h: \Occ \rightarrow \re$ with prescribed boundary conditions  $h(\by) = 0 \ \text{ on } \partial \Oc$.
This infinite-dimensional minimization problem provides a solution \( h \) whose gradient is the least-squares approximation to the guidance field \( \vec{\bv} \). 
A twice differentiable minimizer of \eqref{eq: variation cost}, $h\in C^2(\overline{\Oc})$,  satisfies the associated Euler-Lagrange equation, given by:
\begin{gather}\label{eq: poisson's eq div}
\left \{
    \begin{aligned}
        \Delta h(\by) &= \nabla \cdot \vec{\bv}(\by) & \text{ in } \Omega,\\
        h(\by) &= 0 &  \text{ on } \partial \Omega, \\
    \end{aligned}
    \right.
\end{gather}
where $f(\by)=\nabla \cdot \vec{\bv}(\by)$ denotes the divergence of the guidance field, defining the forcing function. 
We provide a discussion on Euler-Lagrange equations and the fact that solutions to Poission's equation can be realized as unique minimizers of variational problems in Appendix~\ref{appdx: poisson's equation}, specifically, Theorem~\ref{thm: dirichlet principle}. 

A smooth guidance field, \( \bvv \in C^\infty(\Occ;\re^3)\), implies \( \divv \in C^\infty(\Occ)\) in \eqref{eq: poisson's eq div}, and following from Theorem~\ref{thm:regularity}, this leads to a smooth solution, \(h \in C^\infty(\overline \Oc)\). We obtain a smooth guidance field by solving Laplace's equation---the homogenous version of Poisson's equation. 
Specifically, consider the vector field $\vec{\bv} = (v_x, v_y, v_z) :\overline{\Oc} \rightarrow \re^3$, with each component satisfying
Laplace's equation subject to Dirichlet boundary conditions:
\begin{gather}\label{eq: guidance field}
\left \{
    \begin{aligned}
    \Delta v_i(\by) &= 0 & \text{in } \Omega, \\
    v_i(\by) &= b(\by) n_i(\by) & \text{on } \partial \Omega,
    \end{aligned}
    \right.
\end{gather}
for $i \in \{x, y, z\}$. The terms \( n_x, n_y, n_z\) represent the components of the outward unit normal vector \( \hat{\mathbf{n}} = (n_x, n_y, n_z):\partial \Oc \rightarrow \re^3  \) such that $\bvv(\by) = b(\by) \hat{\bn}(\by)$  on \( \partial \Omega \).
The term
\( b : \pOc \to \mathbb{R}_{<0} \) prescribes the outward directional derivative encoding the desired boundary flux:
\begin{align}\label{eq: desired boundary flux}
    \oiint_{\pOc} \vec{\bv}(\by)\cdot \hat{\bn}(\by) \, \mathrm{d}A &=  \oiint_{\pOc} b(\by)\hat{\bn}(\by)\cdot \hat{\bn}(\by) \, \mathrm{d}A \nonumber \\ 
    &=  \oiint_{\pOc} b(\by) \, \mathrm{d}A.
\end{align}
Following from the mean value theorem, solutions to \eqref{eq: guidance field} are smooth \cite{evans2022partial}, meaning $\vec{\bv}$ belongs to the set of functions:
\begin{align}\label{eq: smooth admissible set guidance field}
   \Vc = \left\{ \vec{\bv} \in C^\infty(\Occ; \mathbb{R}^3)  \, \big| \, \vec{\bv}(\by) = b(\by) \hat{\bn}(\by)  \text{ on }  \partial \Omega \right\}.
\end{align}
Fig. \ref{fig: guidance field} depicts an example a guidance field generated via (\ref{eq: guidance field}).
%
%

Due to the decoupled nature of its components, \(\vec{\bv}\) may not be \textit{conservative}, meaning it may not correspond to the gradient of a scalar potential function \cite{marsden2003vector,kellogg2012foundations}. To address this, \eqref{eq: poisson's eq div} ensures we find $h \in C^\infty(\overline{\Oc})$ whose gradient best approximates
$\vec \bv$ by using its divergence, $\divv$, as the forcing function. The boundary flux error is given by $Dh(\by)\cdot\hat{\bn}(\by) - b(\by)$  on $\pOc$. However, the condition $\nabla \cdot \vec \bv(\by) <0 $  may not necessarily hold for all $\by \in \Oc$, which is sufficient to guarantee $h(\by) >0$ in $\Omega$. To remedy this, we introduce the forcing function\footnote{Inspired by the $\mathrm{softplus}$ function, this preserves the smoothness and negativity of $f$, ensuring the solution to \eqref{eq: poisson's eq} remains superharmonic in $\Oc$.
}
:
\begin{align}\label{eq: softplus forcing function}
    f(\by) = -\frac{1}{\beta}\ln(1+e^{-\nabla \cdot \vec \bv(\by)\beta}),
\end{align}
with $\beta >0$. This defines a smooth, negative forcing function  $f \in C^\infty(\Oc; \re_{<0})$, which yields a solution $h \in C^\infty(\Occ; \re_{\geq 0})$ to \eqref{eq: poisson's eq} as established in Theorem~\ref{thm: main thm Safety Value Function}. 
It further follows from Theorem~\ref{thm: dirichlet principle} that this solution is the 
unique minimizer of the
variational problem (dropping dependency on $\by$ for brevity):
\begin{gather}
   \! \min_{h\in \Hc} \left \{ \! J[h] \!=\!\!  \iiint_{\Oc} \frac{1}{2}|D h|^2  \!-\! \frac{h}{\beta}\ln(1 +e^{-\nabla \cdot \vec \bv\beta}) \mathrm{d}V \!\right \},\\
    \Hc =\big\{ h \in C^\infty(\Occ) \, \big | \, h = 0  \text{ on } \partial \Oc \big \}.
\end{gather}
Example safety functions generated with the approach are depicted in Fig.~\ref{fig: forcing funcs}. 
 Our next goal is to establish the forward invariance of the safe set $\Cc$, characterized by the safety function $h$ as in \eqref{eq: safe set characterizations}, for time-parameterized curves $t \mapsto \bvy(t)$ describing the evolution of a dynamical system.

%% file: safety_with_poisson.tex
\section{Safety-Critical Control via Poisson Safety Functions}\label{sec: CBFs from Poisson}
In this section, we show that a safety function $h$ constructed from Theorem~\ref{thm: main thm Safety Value Function} can be used to define a CBF. We establish that $\Cc$ can be rendered forward invariant for first-order systems while a subset of $\Cc$ can be rendered forward invariant for high-order systems.
We focus on systems defined by integrator chains as \eqref{eq: linear output dynamcis1},  with the input appearing at the last layer---note that our method can be extended to classes of systems with 
outputs of nonuniform relative degree, such as those in  \cite{cohen2024constructive,bahati2025control}. For the spatial outputs \eqref{eq: spatial coordinates}, we have $p =3$ in \eqref{eq: linear output dynamcis1}. 
We begin by discussing forward invariance for first-order systems.
\subsection{First Order Systems}
Consider the following system of relative degree $r = 1$:
\begin{align}\label{eq: r =1 system}
     \dot{\bvy} = \bw,
\end{align}
where the state $\bvy \coloneqq \by \in \re^3$. The following proposition establishes the forward invariance of $\Cc$ with respect to \eqref{eq: r =1 system}.
\begin{proposition}\label{lemma: relative degree 1}(Forward Invariance for First-Order Systems) Let $\Omega \subset \re^3$ be an open, bounded and connected set with piecewise smooth boundary $\partial \Oc$. 
Consider the system \eqref{eq: r =1 system} and a safe set $\Cc$ defined as the $0$-super level set of a function $h: \overline \Oc \rightarrow \re$. Suppose that $h$ is the solution to \eqref{eq: poisson's eq} with $f \in C^{0,\alpha}(\Oc;\re_{< 0})$ for some $\alpha \in (0,1)$. Then  given $\bvy(0) = \by_0 \in \Cc$, there exists a locally Lipschitz continuous controller $ \bw = \bk(\bvy)$ such that for any $\gamma >0$: 
\begin{align}\label{eq: hdot single int}
    \dot{h}(\bvy) = D h(\bvy) \cdot \bk(\bvy) \geq -\gamma h(\bvy) \quad \forall \, \bvy \in \Cc.
\end{align}
%
Thus, $\Cc$ is forward invariant and $h$ is a CBF for \eqref{eq: r =1 system}. 
\end{proposition}
\begin{proof}
From Theorem~\ref{thm: main thm Safety Value Function}, we have that $f \in C^{0,\alpha}(\Oc;\re_{< 0})$ implies that $h \in C^{2,\alpha}(\overline{\Omega};\re_{\geq 0} )$ defining a safe set $\Cc = \overline{\Oc}$ as described in \eqref{eq: safe set characterizations}. Furthermore, we have that $D h(\by) \neq 0$ when $h(\bvy)=0$, that is, when $\by \in \partial \Oc = \partial \Cc$. Given the system \eqref{eq: r =1 system} with initial condition $\bvy(0) = \bvy_0 \in \Cc$, the following locally Lipschitz continuous controller \cite{CohenLCSS23,SontagUniversal}:
\begin{equation}\label{eq: sontag}
\begin{aligned}
    \bk(\bvy) &\coloneqq \bk(\by, h(\by), D h(\by)) \\
   &= \lambda(\gamma h(\by), \|Dh(\by)\|^2)D h(\by)^\top \\
    \lambda(a,b) &= \begin{cases} 0, &b = 0\\
    \frac{-a + \sqrt{a^2 +\sigma b^2}}{2b}, & b\neq 0
\end{cases}
\end{aligned},
\end{equation}
with $\sigma>0$ satisfies \eqref{eq: hdot single int} for all $\bvy \in \Cc$ such that when $h(\bvy) = 0$, we have $\bk(\bvy) = \frac{1}{2}\sigma^\frac{1}{2} Dh(\by)^\top $ yielding
$\dot{h}(\bvy) = \frac{1}{2}\sigma^\frac{1}{2} \|D h(\by)\|^2 >0
$. Thus, from Nagumo's theorem, the set $\Cc$ is a rendered forward invariant and $h$ is a CBF for \eqref{eq: r =1 system}. 
\end{proof}

\subsection{CBF Backstepping}

Given a safety function $h$ satisfying Theorem~\ref{thm: main thm Safety Value Function}, we now establish safety for high-order systems  \eqref{eq: linear output dynamcis1} of relative degree $r \geq 2$ by leveraging finite-dimensional backstepping to construct a CBF \cite{taylor2022safe}. CBF backstepping is a design technique that recursively constructs auxiliary controllers for each layer of the full-order system dynamics, providing a systematic approach to designing CBFs for linear systems modeled as integrator chains \cite{cohen2024constructive}. For convenience, we begin by considering systems of relative degree $r= 2$. These results can be extended to relative degree $r\in\mathbb{N}$, as demonstrated in Appendix~\ref{appdx: higher-order cbfs}.

Consider the state $\bvy = \big(\by, \dot{\by}\big) \in \re^6$ and dynamics:
\begin{align}\label{eq: r =2 system}
     \dot{\bvy} = \begin{bmatrix}
         \dot{\by} \\
     \ddot{\by} 
     \end{bmatrix} = \begin{bmatrix}
         \dot{\by} \\
        \bw
     \end{bmatrix}.
\end{align}
Specifically, let  $h \in C^{2+k,\alpha}(\overline{\Oc})$  be the solution to \eqref{eq: poisson's eq} satisfying Theorem~\ref{thm: main thm Safety Value Function}. Define the following function:
\begin{align}\label{eq: second order h_B}
h_\mathrm{B}(\bvy) = h(\by) - \frac{1}{2\mu_1}\|\dot{\by} -\bk_1(\by, h(\by), D h(\by))\|^2,
\end{align}
with $\mu_1 >0$ where $\bk_1 \in C^{2}(\re^3 \times \re \times \re^3;\re^3)$ is an auxiliary controller satisfying:
\begin{align}\label{eq: k1 for r=2}
    D h(\by) \cdot  \bk_1(\by, h(\by), D h(\by)) > -\gamma h(\by),
\end{align}
for all $\by \in \Cc$ and some $ \gamma > 0$ such as \eqref{eq: sontag}. 
The 0-superlevel set of the function $h_\mathrm{B}$ defines the shrunken set:
\begin{align}\label{eq: r=2 safe set}
    \Cc_\mathrm{B} = \left\{\bvy \in \re^6  \, \big| \, h_\mathrm{B}(\bvy) \geq 0 \right \} \subset \Cc \times \re^3.
\end{align}
Since $h_\mathrm{B}(\mb{y}) \leq h(\mb{y}) $ for all $\mb{y}\in \mathcal{C}$, we can ensure that all trajectories that start in $\mathcal{C}_\mathrm{B}$ also remain in $\mathcal{C}$ by rendering $\mathcal{C}_\mathrm{B}$ safe. This is possible because $\vec{\mb{y}} \in \mathcal{C}_\mathrm{B}  \cap \partial\mathcal{C}  $ implies that $\dot{\mb{y}} = \mb{k}_1(\mb{y}, h(\mb{y}), Dh(\mb{y})) $ and therefore, the condition \eqref{eq: k1 for r=2} ensures that $\dot{h}(\by, \dot{\by}) = Dh(\by) \cdot \dot{\by} > 0$ at this point $\vec{\mb{y}}$. 
Taking the time derivative of \eqref{eq: second order h_B} yields:
\begin{align}\label{eq: hdot r=2}
     \dot{h}_\mathrm{B}(\bvy, \bw) = Dh(\by) \cdot \dot{\by} - \frac{1}{\mu_1}(\dot{\by} -\bk_1)^\top (\bw - \dot{\bk}_1),
\end{align}
where $\dot{\bk}_1(\cdot) = \frac{\partial \bk_1}{\partial \by}(\cdot)\dot{\by} \coloneqq \Phi_1(\by, \dot{\by}, Dh(\by), D^2h(\by))$ with:
\[
\frac{\partial \bk_1}{\partial \by}= \frac{\partial \bk_1}{\partial \by} + \frac{\partial \bk_1}{\partial h} \frac{\partial h}{\partial \by}  + \sum_{i \in \{x,y,z\}}\frac{\partial \bk_1}{\partial (\partial_i h)} \frac{\partial (\partial_i h)}{\partial \by} 
.
\]

To construct a locally Lipschitz continuous controller $ \bw=\mb{k}(\bvy)$ that enforces safety, we require the second partial derivatives of $h$, denoted by the Hessian $D^2h$ to be Lipschitz continuous \ie $D^2h \in C^{0,1}(\overline{\Oc}; \re^{3\times 3})$. 
The following lemma establishes conditions for this property to hold.

\begin{lemma}\label{lemma: Lipschitz gradients}(Lipschitz Regularity of $D^rh$) Let $r \in \mathbb{N}$ and $f \in C^{k,\alpha}(\Occ)$ for some $\alpha \in (0, 1)$ and $k \in \mathbb{N}_0$. If $r \leq k+1 $, then the solution \( h \in C^{2+k,\alpha}(\overline \Oc) \) to  \eqref{eq: poisson's eq}  satisfies $D^rh \in C^{0,1}(\overline{\Oc}; \re^{3^r})$.
\end{lemma}
\begin{proof}
    From Theorem~\ref{thm:regularity}, one can verify that the $r$-th partial derivatives \( D^rh \in C^{2+k-r,\alpha}(\overline \Oc;\re^{3^r} )\). For the case when \( r  = k+1 \), we have that $C^{1,\alpha} \subset C^{0,1}$. Thus, it follows that \( C^{2+k-r,\alpha} \subset C^{0,1} \) for \( r \leq k+1 \), which implies that \( D^rh \) is Lipschitz continuous on the set $\overline{\Oc}$.
\end{proof}
The above lemma relies on the inclusion of continuously differentiable functions in Lipschitz continuous function spaces \cite{brezis2011functional}.
Leveraging the above result, it follows from \cite[Theorem 4]{taylor2022safe} that there exists a locally Lipschitz continuous controller that renders \( \Cc_\mathrm{B} \) forward invariant. Specifically, if \( {\bvy}_0 = (\by_0,  \dot{\by}_0) \in \Cc_\mathrm{B}\), then \( {\bvy}(t) \in \Cc_\mathrm{B} \) for all \( t \in I_\mathrm{max}({\bvy}_0) \). Thus, $h_\mathrm{B}$ is a CBF for \eqref{eq: r =2 system}.

\begin{remark}(Hölder Regularity and Control Design)
    The results in \cite{ cohen2024constructive,taylor2022safe} rely on the assumption that $h$ is smooth, \ie $h\in C^\infty(\Occ)$, which holds if $f \in C^\infty(\Occ)$ as established in Theorem~\ref{thm:regularity}. In this paper, we relax this assumption and only require Hölder continuity, $h\in C^{2+k,\alpha}(\Occ)$, therefore, allowing for a wider range of methods for constructing forcing functions that may not be smooth but are instead Hölder continuous, $f\in C^{k,\alpha}(\Occ)$. This relaxed regularity assumption motivates the use of Lemma~\ref{lemma: Lipschitz gradients} to guarantee the existence of a Lipschitz continuous controller on $\Occ$.
\end{remark}

%% file: demonstrations.tex
\begin{figure*}[t!] 
    \centering    \includegraphics[width=1.\linewidth]{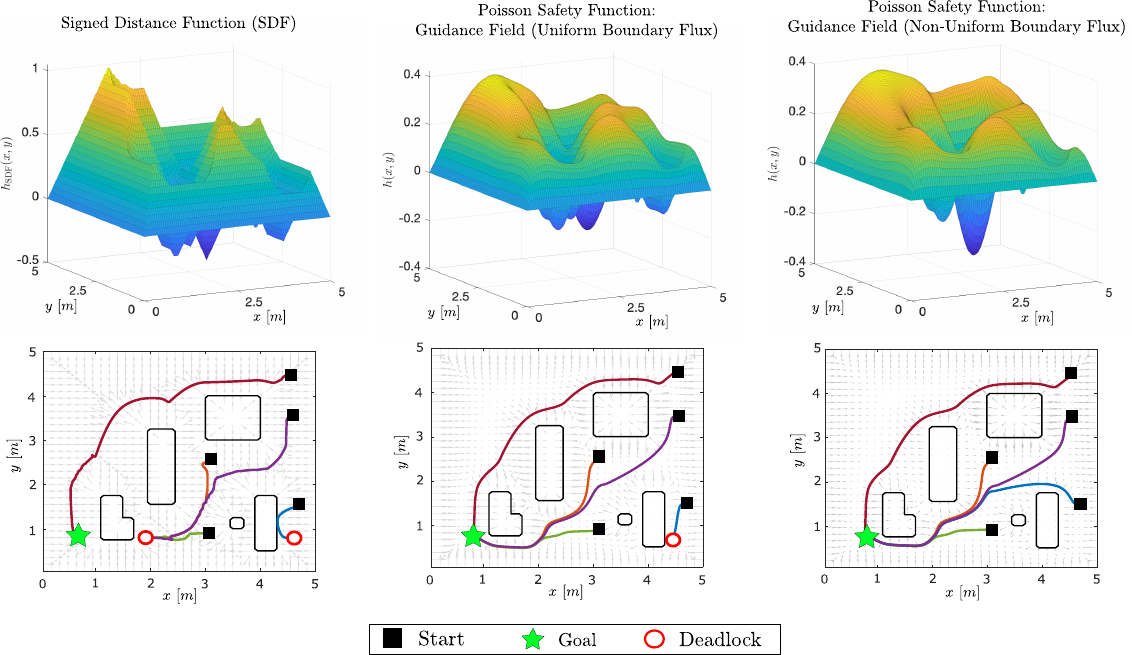}
    \caption{Double integrator simulations using safety filters synthesized from: \textbf{[left]} Signed Distance Function \eqref{eq: SDF}; and \textbf{[middle and right]} the Poisson Safety Function, constructed with the forcing function \eqref{eq: softplus forcing function} with the guidance field \eqref{eq: guidance field} where the boundary conditions use \textbf{[middle]} a uniform boundary flux $b(\by) =-1$ for all $\by \in \pOc$ and \textbf{[right]} a non-uniform boundary flux $b: \pOc \to \mathbb{R}_{<0}$, assigning different flux values across regions of the boundary associated with different obstacles. 
    Sharp ridges in the SDF surface introduce gradient discontinuities, which lead to oscillations in the resulting trajectories---an issue which does not arise with Poisson safety functions due to their classical regularity (\ie differentiability) properties. Furthermore, in contrast to SDFs with fixed gradients, the guidance field promotes the manipulation of boundary flux values via the assignment of $b$ in \eqref{eq: guidance field}, enabling the ability to encode gradients customized to specific obstacles, which we are unable to do with traditional SDFs. This flexibility helps yield trajectories that avoid undesired equilibria \ie “deadlocks”, in \textbf{[right]}.}  
    \vspace{-5mm}
    \label{fig: sdf vs poisson}
\end{figure*}
\section{Demonstrations}
To demonstrate the proposed algorithm's effectiveness in synthesizing safety filters for autonomous systems in complex environments, we perform numerical simulations and hardware experiments. We solve Poisson's equation \eqref{eq: poisson's eq} numerically on a discrete spatial grid by employing a finite difference scheme, specifically the Successive Overrelaxation (SOR) method \cite{strikwerda2004finite}. 
\subsection{Simulations: Double Integrator}
We define a 2D occupancy map defined by an open, bounded and connected domain $\Oc$ where $\pOc$ characterizes obstacle surfaces, and consider the 2D double integrator model of the form \eqref{eq: r =2 system} with state $\bvy = \big(\by, \dot{\by}\big) = \big(x, y, \dot{x}, \dot{y}\big)\in \re^4$. We consider the task of stabilizing \eqref{eq: r =2 system} to a goal position, from various initial conditions, with a nominal PD controller $\bk_\nom(\bvy)$, while avoiding collisions with the obstacle surfaces on $\partial \Oc$. To achieve this, we consider the solution $h$ to \eqref{eq: poisson's eq} with the forcing function \eqref{eq: softplus forcing function} and construct a CBF via backstepping \eqref{eq: second order h_B}. Using this CBF, we synthesize a safety filter to adjust $\bk_\mathrm{\nom}$ to ensure safety. The resultant trajectories for various initial conditions are depicted in Fig~\ref{fig: sdf vs poisson}. We compare the behavior of this safety filter to one constructed via an SDF:
\begin{align}\label{eq: SDF}
    h_\mathrm{SDF}(\by) = 
    \begin{cases}
        \mathrm{dist}(\by, \partial \Oc), & \by \in \Occ , \\
        -\mathrm{dist}(\by, \partial \Oc), & \by \notin \Occ, \\
    \end{cases}
\end{align}
where we also employ backstepping to construct a CBF. 

Traditional SDFs have fixed gradient magnitudes, leading to flat surfaces (\ie without curvature) in the resulting safety function, as seen in Fig.~\ref{fig: sdf vs poisson}. They also possess ridges with discontinuous gradients (of opposite sign) within unoccupied regions of the domain, which, unlike Poisson safety functions, result in undesirable chattering behavior when employed in safety filter synthesis. Additionally, Poisson safety functions constructed using the forcing function \eqref{eq: softplus forcing function} and the guidance field \eqref{eq: guidance field} enable the assignment of arbitrary boundary flux values $b$ via the boundary conditions in \eqref{eq: guidance field}. Consequently, different flux values can be specified across regions of the boundary associated with different obstacles. This provides the flexibility to manipulate gradients around obstacles, which, among other benefits, aids in yielding trajectories that avoid undesired equilibria as demonstrated in Fig.~\ref{fig: sdf vs poisson}.

\begin{remark}(Distance Functions via Elliptic PDEs) 
While prior work has approximated SDFs using elliptic PDEs \cite{belyaev2015variational}, often relying on semilinear equations such as the \textit{screened} Poisson equation \cite[Thm 2.3]{varadhan1967behavior},
our method intentionally does not produce an SDF. Instead, we leverage linear elliptic PDEs with classical regularity results as established in Theorem~\ref{thm:regularity}, which are critical in the synthesis of safety filters.   
\end{remark}
\subsection{Hardware Experiments}

To demonstrate the practical performance of our proposed algorithm in synthesizing safe sets, we applied it to several collision avoidance scenarios using Unitree's Go2 quadruped and G1 humanoid robots. We leverage a reduced order model (ROM) hierarchical framework \cite{TamasRAL22, CohenARC24} for these platforms based on (\ref{eq: r =1 system}). These ROMs assume that the robotic system has a \textit{sufficient} low-level tracking controller, enabling the assignment of safe velocity commands without modification of unique low-level locomotion controllers. 

First, we perceive and segment the environment using a fixed RGB camera and the Meta SAM2 \cite{ravi2024sam2segmentimages} segmentation algorithm. Next, we generate a 2D occupancy map, buffered for robot size. Finally, we solve \eqref{eq: guidance field} to generate a smooth guidance field, and, using \eqref{eq: softplus forcing function}, we solve \eqref{eq: poisson's eq} for the Poisson safety function. 
For dynamic environments, we improve the computational speed of our PDE solver by warm-starting each PDE solution with the previous safety function, producing solve times of $0.2 - 0.3$ ms. After considering the entire processing chain, we update the Poisson safety function $h$ online at approximately 10 Hz. 

We employ this safety function in a CBF-based safety filter for a single integrator ROM (\ref{eq: r =1 system}) to produce safe velocity commands that are tracked by a low-level controller. Robot states are estimated by an OptiTrack motion capture system. 

\subsubsection{\textbf{Complex Static Environment -- Quadruped}}

For the first experiment, we constructed a static environment. We initialized the quadruped at three different starting locations and commanded it to walk to a fixed goal point. In each case, the nominal controller attempted to drive the system directly to the goal without safety considerations. Meanwhile, the CBF-based safety filter modified these nominal commands. The results corresponding to this experiment are depicted in Fig~\ref{fig: experimental results2}.

From these results, it is clear that the Poisson safety function enabled collision avoidance without hindering the nominal objective. The control inputs (Fig~\ref{fig: experimental results2} (bottom left/middle)) show how the safety-filtered velocity commands deviated from nominal, as necessary. The evaluated value of $h(x,y)$ maintained positivity throughout the duration of all three experiments, confirming that the robot never departed from the safe set $\mathcal{C}$.  

\subsubsection{\textbf{Dynamic Environment -- Quadruped}}

To further demonstrate the real-time utility of our approach, we consider a dynamically changing environment. As mentioned earlier, we compute the Poisson safety function online at approximately 10 Hz; however, this time we numerically incorporate the temporal change in $h$ to approximate its time derivative via:
\begin{align}
     \pdv{h}{t}(t, \by) \approx \frac{h(t_k,\by) -  h(t_{k-1},\by)}{t_k - t_{k-1}}.  
\end{align}
\begin{figure*}[t!] 
    \centering    \includegraphics[width=1.\linewidth]{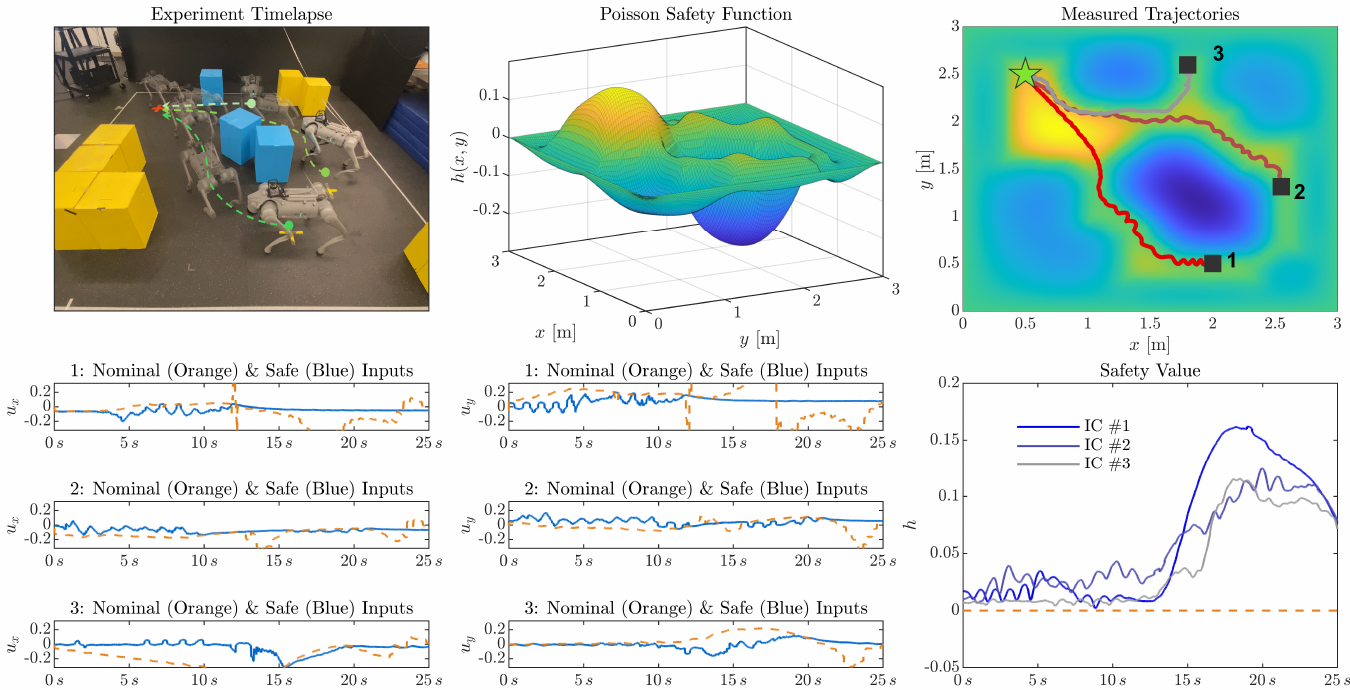}
    \caption{Hardware experiments demonstrating \textit{Poisson safety functions} for safety filtering. [Top Left] A timelapse showing the motion of the Go2 quadruped during the 25-second tracking experiment, starting from three difference initial conditions (ICs). [Top Middle] The Poisson safety function  constructed from real-time segmented image data. [Top Right] The resultant safe trajectories for each IC. [Bottom Left/Middle] The nominal (orange) and CBF safety-filtered (blue) velocity commands, in m/s, sent to the robot. [Bottom Right] The evaluated value of the Poisson safety function over the course of each experiment. This value remains above zero, confirming that safety was maintained.}  
    \vspace{-5mm}
    \label{fig: experimental results2}
\end{figure*}
We commanded a Go2 quadrupedal robot to track a nominal sinusoidal position reference; meanwhile, we introduced two obstacles (i.e., a desk chair and a yellow box) into the environment. By dynamically rearranging these obstacles, we tested the real-time efficacy of our safety-critical architecture. Fig~\ref{fig: experimental results} shows the dynamic performance. 

Examining the value of $h$ during the experiment, it can be observed that the robot effectively employed its safety filter to avoid collisions. The value of $h$ remained positive, confirming that the safe set $\mathcal{C}$ was rendered forward invariant.

\subsubsection{\textbf{Dynamic Environment -- Humanoid}}

We implemented an identical safety-critical controller on the G1 humanoid to show the versatility of our algorithm. We commanded a fixed position reference and introduced a moving obstacle into the environment. The results are presented in Fig~\ref{fig: experimental results}. 

During the humanoid experiment, the CBF-based safety filter was effective in preventing collisions with the dynamic obstacle; however, the robot experienced minor safety violations, corresponding to moments when $h$ was briefly negative. The difference in performance on the quadruped highlights the importance of accurate ROMs when designing safety-critical controllers~\cite{CohenARC24}. In both cases, we filtered velocity commands (single-integrator \eqref{eq: r =1 system}) to enforce safe set forward invariance. This is an accurate ROM for the quadruped, as the low-level locomotion controller effectively tracks velocity inputs with minimal error. Unfortunately, this ROM is less accurate for the humanoid, which struggles to track velocity commands due to model mismatch. The lag in velocity tracking manifested in modest departures from the safe set (\ie brief moments when $h$ was negative). 

\begin{remark}
(Real-time Computation)
     For the hardware experiments in this paper, we consider a $3\times3$ $\text{m}^2$ experimental environment, and we solve the 2D Poisson equation on an $N$$\times$$N$ discretized grid with $N=120$. We apply a SOR finite difference scheme \cite{strikwerda2004finite}, using \textit{checkerboard} iterations for GPU parallelization. Parallel SOR iterations scale with $\sqrt{N}$, and computations per iteration scale with $N^2$ ($N^3$ in 3D). On a GeForce RTX 4070 GPU, we solve the PDE in $0.2-0.3$ ms with $10^{-4}$ residual tolerance. 
 \end{remark}

\section{Limitations}
A fundamental limitation of the proposed algorithm (and a limitation of all non-predictive safety filters) is that such safety-critical controllers may introduce undesired equilibria. These equilibria can manifest as ``deadlocks", where the system becomes trapped by obstacles and fails to achieve nominal tracking. Such problems are typically solved by introducing a navigational layer to the nominal controller to avoid regions in the state-space that produce these undesired equilibria. 

Another limitation of the proposed approach is that the extension of the solution $h$, defined on $\Occ$, to the unsafe regions $\Gamma_i$, via Poisson's equation is guaranteed only to be Lipschitz continuous across safe and unsafe regions (\ie across the boundary). The reason is that solving Poisson's Dirichlet problem results in unique solutions in each region ensuring $h(\by)= 0$ on $\pOc$ (hence, continuity across $\pOc$). As a result, the directional derivatives from the safe and unsafe regions are not guaranteed to be equal in magnitude on $\pOc$. However, from Hopf's lemma, one can verify that they have the same sign \ie positive flux from each obstacle $\Gamma_i$ and negative flux into the safe region $\Oc$. Thus, if safety violations occur, there exists a locally Lipschitz continuous controller enforcing attractivity to the safe set for a single integrator, but may require a less regular controller \cite{coron2007control} for higher order systems. One way of obtaining a more regular $h$ across these two regions is by a smooth extension of the solution $h$ defined on $\Occ$ into $\Gamma_i$. These extensions however, may require detailed geometrical characterizations of obstacle interiors, whereas Poisson's equation itself only requires information about the obstacle surfaces. 
Another approach is to employ mollifiers \cite{gilbarg1977elliptic} to regularize $h$ across these regions.

%% file: conclusion.tex
\begin{figure*}[t!] 
    \centering    \includegraphics[width=1.\linewidth]{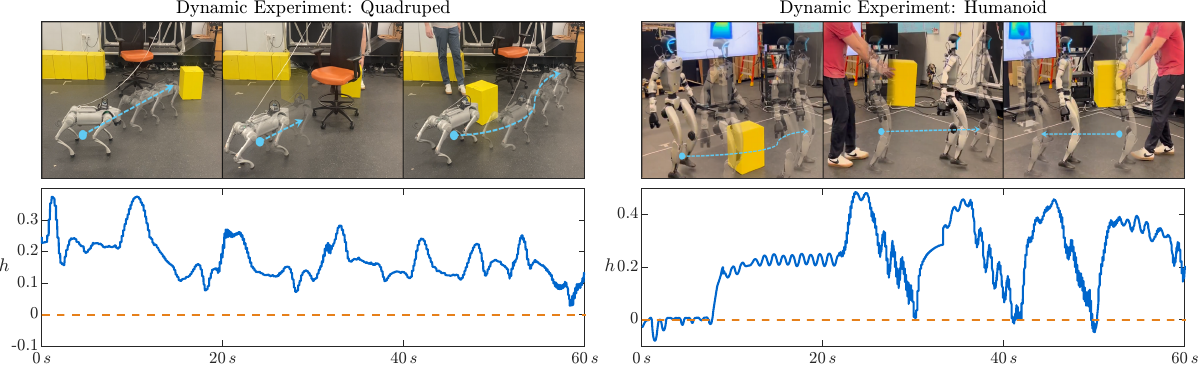}
    \caption{Hardware experiments demonstrating dynamic behavior on quadrupedal and humanoid robots. [Top Left] A timelapse showing the Go2 quadruped successfully avoiding dynamic complex obstacles in real-time. [Bottom Left] The value of the Poisson safety function $h$ during the experiment. The positivity of $h$ certifies that safety was maintained. [Top Right] A timelapse of the G1 humanoid during a dynamic collision avoidance demonstration.  [Bottom Right] The value of $h$. Although no collisions occurred, the robot briefly left the safe set $\mathcal{C}$.} 
    \vspace{-4mm}
    \label{fig: experimental results}
\end{figure*}
\section{Conclusion} 
\label{sec:conclusion}
We have presented an algorithm for generating safe sets for complex environments by solving Poisson's equation. Our method uses environmental perception data to define a domain, on which we impose boundary conditions, and formulate a boundary value problem for Poisson's equation with a novel forcing function. The resulting solution is a safety function---a functional representation of the environment that characterizes safety and defines the safe set. This safety function is also the unique minimizer of a variational problem, and has desirable regularity properties ideal for safety-critical control design in robotic applications. The real-time efficacy of the proposed algorithm enables its use in dynamically changing environments. To validate this, we utilize the safety function to synthesize CBF-based safety filters, and demonstrate their effectiveness in achieving collision avoidance on quadruped and humanoid robots navigating in complex, dynamic environments. 
\section*{Acknowledgments}
\noindent
This work was supported by BP. The authors thank Ryan K. Cosner and Max H. Cohen for discussions and their insightful perspectives in the development and presentation of this work, and Lihzi Yang for the humanoid experimental setup.

%% file: appendix_poisson.tex
\newpage
\section{Poisson's equation}\label{appdx: poisson's equation}

This appendix reviews key concepts from second-order linear elliptic PDEs, emphasizing details relevant to our results. The notation and expositions follow \cite{gilbarg1977elliptic, jost2012partial, evans2022partial}.

\subsection{Dirichlet Problem for Poisson's equation}
Let $\Oc \subset \re^3$ be an open, bounded and connected set 
with smooth boundary $\pOc$, such that $\Occ = \Oc \cup \pOc$. Poisson's equation is a second-order linear elliptic PDE defined by:
\begin{align}\label{eq: poison bckgrd}
\Delta h(\by) = f(\by) \quad \text{ in } \Oc,
\end{align}
where $\by = (x,y,z) \in \re^3$, \( \Delta = \frac{\partial^2 }{\partial x^2} + \frac{\partial^2 }{\partial y^2} + \frac{\partial^2 }{\partial z^2}\) is the \textit{Laplacian} and $f: \Oc \rightarrow \re$ is a given \textit{forcing} function. If $f(\by)=0$ for all $\by \in \Oc$, \eqref{eq: poison bckgrd} is a homogeneous elliptic PDE called Laplace's equation.
To completely determine solutions to \eqref{eq: poison bckgrd}, boundary conditions must be specified to characterize the behavior of a solution on $\partial \Oc$ \cite{evans2022partial}. \textit{Dirichlet} boundary conditions specify the solution's value, while \textit{Neumann} boundary conditions specify the normal derivatives on $\partial \Oc$.
In this work, we consider solutions $h: \Occ \rightarrow \re$ satisfying Poisson's equation subject to Dirichlet boundary conditions:
\begin{gather}\label{eq: poisson's eq appndx}
\left \{
    \begin{aligned}
        \Delta h(\by) &= f(\by)& \text{ in } \Omega,\\
        h(\by) &= 0 &  \text{ on } \partial \Omega. \\
    \end{aligned}
    \right.
\end{gather}
A \textit{classical} solution is a twice continuously differentiable function $h \in C^2(\Occ)$ satisfying \eqref{eq: poisson's eq appndx}. 
The sign of the values of the forcing function $f(\cdot)$ in $\Oc$ leads to further characterizations of the solution to \eqref{eq: poisson's eq appndx}, provided in the following definition.
\begin{definition}
    Let $\Omega \subset \re^3$ be an open, bounded and connected set. A function $h \in C^2(\Oc)$ satisfying $\Delta h(\by) = 0 \ (\geq 0, \leq 0)$ in $\Oc$ is called \textit{harmonic} (\textit{subhamornic,} \textit{superharmonic}) in $\Oc$. 
\end{definition}
In the following subsections, we discuss the structural properties of solutions to the boundary value problem \eqref{eq: poisson's eq appndx} relevant to the developments presented in this paper.

\subsection{Gauss's Divergence Theorem}
A solution $h \in C^2(\Occ)$ to \eqref{eq: poisson's eq appndx} represents the density of some quantity across the domain $\Occ$. The \textit{flux} of the vector field $Dh$ is the net flow of $h$ across the boundary of the domain, and is given by the surface integral of the normal component of $Dh$ (\ie outward directional derivative of $h$) as follows:
\begin{align}
    \oiint_{\partial \Omega} D h(\by) \cdot \hat{\mathbf{n}}(\by) \, \mathrm{d} A,
\end{align}
 where  $\hat{\mathbf{n}} : \pOc \rightarrow \re^3$ is the outward-pointing unit normal on $\pOc$, $\mathrm{d} A$ is the surface element of $\partial \Oc$ and:
\begin{align}\label{eq: Hopf Lemmainequality1}
   D h(\by) \cdot \hat \bn(\by) = \lim_{\delta \rightarrow 0^+}\frac{h(\by + \delta \hat{\bn}(\by)) - h(\by)}{\delta},
\end{align}
for all $\by \in \pOc$. Given this property, we highlight a fundamental theorem associated with vector fields on closed domains.

\begin{theorem}(Gauss's Divergence Theorem \cite{marsden2003vector}) Let $\Oc \subset \re^3$ be an open, bounded and connected set with smooth boundary $\pOc$. Let $\hat{\mathbf{n}} : \pOc \rightarrow \re^3$ be an outward unit normal vector on $\pOc$ and suppose $h \in C^2(\Occ)$, ensuring the vector field $Dh \in C^1(\Occ;\re^3)$. Then:
\begin{align}
 \iiint_\Omega \Delta h(\by) \, \mathrm{d} V = \oiint_{\partial \Omega} D h(\by) \cdot \hat{\mathbf{n}}(\by) \, \mathrm{d} A,
\end{align} 
where $\Delta h = \mathrm{div}({Dh}) = \nabla  \cdot Dh $ is the divergence of $Dh$, $\mathrm{d} V \coloneq \dx\dy \mathrm{d}z$ denotes the volume element of $\Oc$ and $\mathrm{d} A$ is the surface element of $\partial \Oc$. 
\end{theorem}
Gauss's divergence theorem states that the flux of a vector field through a closed surface equals the integral of the divergence of that vector field over the volume enclosed by the surface. This condition is particularly helpful in the designing forcing functions $f$, as demonstrated in Section~\ref{sec: forcing func construction}.

\subsection{Maximum and Minimum Principles}
Second derivatives provide information about the extremal values of a function. 
For Poisson's equation, the
\textit{strong} maximum and minimum principles state that a subhamornic (superharmonic) function $h$ attains its maximum (minimum) in $\Oc$ if and only if $h$ is constant, while \textit{weak} maximum and minimum principles further state that bounded subhamornic (superharmonic) functions attain their extremal values on the boundary of $\Oc$, as we highlight in the following theorem.

\begin{theorem}(Weak Maximum and Minimum Principle \cite{protter2012maximum})\label{thm: weak maximum} 
Let $\Omega \subset \re^3$ be an open, bounded, connected set.
If a function $h \in C^2(\Occ)$ satisfies $\Delta h(\by) \geq 0 \ (\leq 0)$ in $\Oc$, then the maximum (minimum) value of $h$ is on the boundary $\pOc$. That is:
\begin{align}
    \Delta h(\by) \geq 0 \text{ in } \Oc \implies \max_{\by \in \Occ} h(\by) = \max_{\by \in \pOc} h(\by), \\
    \Delta h(\by) \leq 0 \text{ in } \Oc \implies \min_{\by \in \Occ} h(\by) = \min_{\by \in \pOc} h(\by). 
\end{align}

\end{theorem}
The above theorem follows from the mean value property \cite{evans2022partial}, which states that the value of a function at any point in $\Oc$ is equal to the average over a ball 
centered around that point. 
A key result from Theorem~\ref{thm: weak maximum} is Hopf's lemma, which gives a  statement about the sign of the normal derivatives of $h$ on $\partial \Oc$. To facilitate this, we introduce a key property that a domain must satisfy.

\begin{definition}(Interior Sphere Condition \cite{evans2022partial})\label{def: interior sphere property}
An open set $\Omega$ satisfies the \emph{interior sphere condition} at $\by_0 \in \partial \Omega$, if there exists an open ball $B_{\rho}(\by) \subset \Omega$ of radius $\rho >0$ and $\by \in \Omega$ such that 
$ \partial B_{\rho}(\by) \cap \partial \Omega = \{\by_0\}$.
\end{definition} 

The above definition states every point on $\partial \Omega$ can be touched from inside with a closed ball contained in $\overline{\Omega}$ as shown in Figure~\ref{fig: interior sphere}. Notice, if the above condition holds for all points on $\pOc$, then $\Oc$ cannot have outward-pointing ``corners", though it may have inward-pointing ones.
In this work, we assume smooth boundaries, which guarantees the above condition holds and that the gradient $Dh$ exists at each point on $\partial \Oc$ as in \eqref{eq: Hopf Lemmainequality1}. However, this analysis can be readily extended to nonsmooth boundaries under appropriate conditions \cite{grisvard2011elliptic, borsuk2006elliptic}.
Following from this, we have the following lemma for the directional derivatives on $\pOc$.
\begin{lemma}(Hopf's Lemma \cite{protter2012maximum})\label{lemma: Hopf's} Consider an open, bounded and connected set  \( \Omega \subset \mathbb{R}^3 \) 
with smooth boundary $\pOc$,
and let \( h \in C^2(\Occ)\) such that $\Delta h \geq 0 \ (\leq 0)$ in $\Omega$.   
If \( \by_0 \in \partial \Omega \) is an extremal point of $h$ as in Theorem~\ref{thm: weak maximum}, then we have:
\begin{align}\label{eq: hopf's inequalities}
  h(\by_0) = \max_{\by \in \Occ} h(\by) \implies D h(\by_0) \cdot \hat \bn(\by_0) \geq 0,\\ h(\by_0) = \min_{\by \in \Occ} h(\by) \implies D h(\by_0) \cdot \hat \bn(\by_0) \leq 0,
\end{align}
where \( \mathbf{\hat{n}}(\by_0) \) is the outward-pointing unit normal to $\partial \Omega$ at \( \by_0 \). 
\end{lemma}
The maximum and minimum principles can be used to establish the uniqueness of a classical solution $h \in C^2(\Occ)$ to \eqref{eq: poisson's eq appndx}, assuming the solution exists \cite[Theorem 2.4]{gilbarg1977elliptic}. However, these principles do not address the existence of a classical solution. We devote the next subsection to this discussion.
\begin{figure}[t!] 
    \centering    \includegraphics[width=.87\linewidth]{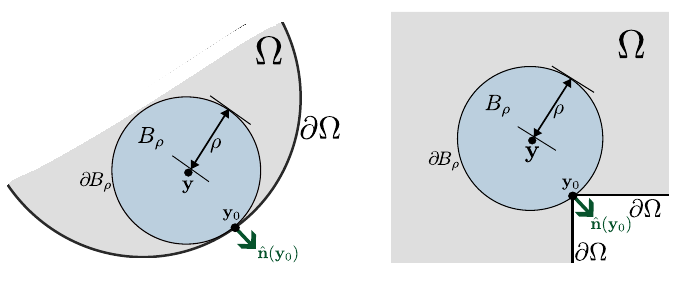}
    \vspace{-3mm}
    \caption{Interior sphere condition for \textbf{[left] }smooth boundary and \textbf{[right]} Lipschitz boundary.}
    \label{fig: interior sphere}
    \vspace{-7mm}
\end{figure}
\subsection{Existence and Regularity of Classical Solutions}
The existence of classical solutions to \eqref{eq: poisson's eq appndx} relies on the properties of the forcing function $f$.
Naturally, one would expect the solution $h$ to \eqref{eq: poisson's eq appndx} to be ``twice more differentiable" than $f$. However, this is not always true (except for the one-dimensional case), and often depends on the function space specifying the regularity of $f$ \cite{gilbarg1977elliptic,jost2012partial}. 
A sufficient condition for \eqref{eq: poisson's eq appndx} to admit a classical solution on $\overline \Oc$ is that $f$ is Hölder continuous, presented in the following definition.
\begin{definition}(Hölder Continuity~\cite{fiorenza2017holder})
A function $f \in C^{0,\alpha}(\Omega)$ is said to be Hölder continuous on $\Oc$ with order $\alpha$ if for some $0 < \alpha < 1$, there exists $M >0$ such that:
\[|f(\by_1) - f(\by_2)| \leq M \|\by_1 - \by_2\|^\alpha \quad \forall \by_1, \by_2 \in \Omega.\]
We call \( C^{0,\alpha}(\Omega) \) the space of Hölder continuous functions of order $\alpha$, with $\alpha = 1$ denoting Lipschitz continuous functions. 
\end{definition}

Specifically, \( C^{k,\alpha}(\Oc) \) is the space of \( C^{k}(\Oc) \) functions whose partial derivatives of order $k \in \mathbb{N}_0$ belong to  \( C^{0,\alpha}(\Oc)\). That is, $\alpha \in (0, 1)$ describes the regularity of the $k$-th partial derivatives of $f$. We further have that $C^{k,0}(\Oc) = C^k(\Oc)$. Furthermore, Hölder continuous functions are uniformly continuous, thus, $f \in C^{k,\alpha}(\Occ)$ denotes a function $f \in C^{k,\alpha}(\Oc)$ whose partial derivatives up to order $k$ admit continuous extensions
to $\Occ$. 
The following theorem guarantees that all second derivatives of $h$ lie in the same function space as $f$.

\begin{theorem}(Classical Elliptic Regularity \cite{gilbarg1977elliptic})\label{thm:regularity}
    Let $\Omega \subset \re^3$ be an open, bounded and connected set with smooth boundary $\partial \Oc$. Suppose $f \in C^{k,\alpha}(\overline \Omega)$ for some $ 0 < \alpha < 1$ and $k \in \mathbb{N}_0$. Then there exists a unique solution $h \in C^{2+k, \alpha}(\overline \Oc)$ to \eqref{eq: poisson's eq appndx}. Furthermore, $f \in C^{\infty}(\overline \Omega) \implies h \in C^{\infty}(\overline \Omega)$.
\end{theorem}
The interior regularity of $h$ on $\Oc$ is determined by the regularity of $f$, and global regularity on $\overline{\Oc}$ follows by extending the interior regularity to the boundary, given that the boundary is (sufficiently) smooth as we assumed. Schauder theory \cite[Chap.~6]{gilbarg1977elliptic} provides classical regularity results for general linear elliptic PDEs in Hölder spaces. In the above theorem, the existence of a classical solution to \eqref{eq: poisson's eq appndx} with a Hölder continuous $f$ is guaranteed by Kellogg's Theorem (see \cite[Theorem 6.14]{gilbarg1977elliptic}) and \cite[Theorem 13.3.1]{jost2012partial}. For interior higher-order regularity, refer to 
\cite[Theorem 6.17]{gilbarg1977elliptic}, while for global higher-order regularity, see \cite[Theorem 6.19]{gilbarg1977elliptic}. 
\begin{remark} The above theorem does not hold for $\alpha =0$, and  $\alpha =1$, that is, $h$ is not guaranteed to be twice differentiable if $f$ is merely continuous or even Lipschitz continuous, and may require further technical assumptions and developments to establish this. For an introduction to elliptic regularity theory, refer to \cite{fernandez2022regularity}; for elliptic equations in Hölder spaces, see \cite{krylov1996lectures}.
\end{remark}
\subsection{Optimality}
 A solution to Poisson's equation can also be recognized as the minimizer of an appropriate functional in the appropriate function space. By recognizing that the solution to a PDE is the critical point of an associated cost functional, proving the existence, uniqueness (e.g. via convexity arguments), and regularity of a minimizing solution can, in some cases, be easier than working directly with the PDE. We have the following theorem for classical solutions to Poisson's equation.

\begin{theorem}\label{thm: dirichlet principle}(Dirichlet's Principle \cite{evans2022partial,dacorogna2024introduction})
Let $\Oc$ be an open, bounded and connected set with smooth boundary $\pOc$. Suppose $h \in C^2(\Occ)$ is the solution to the Dirichlet problem for Poisson's equation \eqref{eq: poisson's eq appndx}. Then $h$ is the unique minimizer of the variational problem:
\begin{gather}\label{eq: variation main}
   \min_{h \in \Hc} \left \{J[h] =  \iiint_{\Oc} \ \frac{1}{2}|D h(\by)|^2 + h(\by)f(\by)  \ \mathrm{d}V \right \},
\end{gather}
in the set of admissible functions:
\begin{gather}\label{eq: admissible set main}
\Hc = \{ h \in C^2(\Occ) \, \big | \, h(\by) = 0 \text{ on } \partial \Oc\}.
\end{gather}
 Conversely, if $h \in \Hc$ satisfies \eqref{eq: variation main}, then it solves the boundary value problem \eqref{eq: poisson's eq appndx}. 
\end{theorem}

The above theorem states that $h \in \Hc$ being the solution to the PDE \eqref{eq: poisson's eq appndx} is equivalent to the statement that $h$ minimizes $J[\cdot]$. This follows from the fact that a necessary condition for optimality of a $C^2(\overline{\Oc})$ minimizer of $J[\cdot]$ is that it satisfies the Euler-Lagrange equation \cite{dacorogna2024introduction}, which, in the case of \eqref{eq: variation main}, corresponds Poisson's equation \eqref{eq: poisson's eq appndx}. The uniqueness of a solution to \eqref{eq: poisson's eq appndx} certifies it as the unique minimizer of $J[\cdot]$. 

The main challenge of Dirichlet's principle is that it is not obvious whether \eqref{eq: variation main} attains its minimum in the chosen class of functions in the set \eqref{eq: admissible set main}. The above theorem is an ``inverse" optimal statement as we already assumed the knowledge of a $C^2(\Occ)$ solution to Poisson's equation \eqref{eq: poisson's eq appndx}, which, as earlier mentioned, relies on the regularity properties of $f$. Therefore, for forcing functions $f$ that are not Hölder continuous, it is useful to consider solutions in less regular function spaces, which we briefly discuss next.

\subsection{Weak Solutions}

For less regular $f$, which may be discontinuous or lack continuous derivatives, solutions $h$ to \eqref{eq: poisson's eq appndx} need not be continuous or even 
continuously differentiable. Moreover, minimizers of \eqref{eq: variation main} may not exist in \eqref{eq: admissible set main}. In such cases, one would prove the  existence of a \textit{generalized} or \textit{weak} solution to allow for wider class of candidate solutions \cite{evans2022partial, brezis2011functional}. Typically, weak solutions satisfy a criterion---often based on integration by parts---that is necessary and sufficient for a differentiable function to solve a PDE, but does not require differentiability.  Thus, every classical solution is also a weak solution. A common condition for less regular $f$ that yields weak solutions is integrability, where $f \in L^p(\Oc)$ for $ p \in (1,\infty)$. For further details on $L^p$ regularity theory for elliptic equations, see \cite{ gilbarg1977elliptic, jost2012partial, fernandez2022regularity, caffarelli1995fully}. 

To establish existence of a classical solution, it is often easier to first prove the existence of a weak solution and then prove sufficient regularity such that this weak solution is also a classical solution. Consequently, in variational methods \cite{dacorogna2024introduction,dacorogna2007direct}, a minimizer $h$ of \eqref{eq: variation main} has to
satisfy the corresponding Euler–Lagrange equation in its \textit{weak form}\footnote{Weak forms (and weak solutions) are employed in Finite Element Method (FEM) algorithms \cite{jost2012partial} to obtain approximate solutions to \eqref{eq: poisson's eq appndx}.}, and if $h$
is sufficiently regular, also in the classical sense \eqref{eq: poisson's eq appndx}.  

In this paper, we focus on classical solutions to  \eqref{eq: poisson's eq appndx} as they provide continuous derivatives which are convenient for control design in robotic applications, and present our results for Hölder continuous forcing functions $f$. However, our results can be generalized to less regular $f$ that yield weak solutions, provided that a proof of regularity---to obtain classical solutions---can be established.

\subsection{ Equivalence of Variational Problems}
We provide an equivalent formulation of the variational problem associated with \eqref{eq: variation cost} to show that a $C^2(\Occ)$ minimizer of $I[\cdot]$ is also the minimizer of a functional with the structure $J[\cdot]$ in Theorem~\ref{thm: dirichlet principle}.
In particular, we have the following:
\begin{align}\label{eq: alternate cost}
    I[h] &=  \iiint_{\Oc} \ \frac{1}{2}|D h  - \vec{\bv} |^2  \ \mathrm{d}V  \\
    &=  \iiint_{\Oc} \ \frac{1}{2}|D h |^2  - D h  \cdot 
 \vec{\bv}  + \frac{1}{2}|\vec{\bv} |^2  \ \mathrm{d}V  \\
   &=\iiint_{\Oc} \ \frac{1}{2}|D h |^2  + h\nabla \cdot \vec{\bv}  + \frac{1}{2}|\vec{\bv} |^2  \ \mathrm{d}V \label{eq: alternate cost last line}
\end{align}
where we used the product rule $\nabla \cdot (h\vec{\bv}) = D h \cdot \vec{\bv} + h\nabla \cdot \vec{\bv}$ and Gauss's divergence theorem for the cross-terms to yield:
\begin{align}
    \iiint_{\Oc} D h \cdot \vec{\bv}  \   \mathrm{d}V
    &= \iiint_{\Oc}  \nabla   \cdot  (h\vec{\bv}) - h\nabla \cdot \vec{\bv} \mathrm{d}V
    \\
   &= \oiint_{\partial \Oc} h\vec{\bv} \cdot \hat{\bn} \ \mathrm{d}A - \iiint_{\Oc}h\nabla \cdot \vec{\bv} \ \mathrm{d}V \\
   &= - \iiint_{\Oc}h\nabla \cdot \vec{\bv} \ \mathrm{d}V. 
\end{align}
since $h(\by)=0$ on $\pOc$. 
Computing the Euler-Lagrange equation associated with \eqref{eq: alternate cost last line} yields \eqref{eq: poisson's eq div}, thus verifying that a $C^2(\Occ)$ minimizer of $I[\cdot]$ is also a minimizer of the functional:
\begin{align}
    J[h] =  \iiint_{\Oc} \ \frac{1}{2}|D h(\by)|^2  + h(\by) \divv(\by)
 \ \mathrm{d}V,
\end{align}
since the last term in \eqref{eq: alternate cost last line} does not affect \eqref{eq: poisson's eq div}. That is, \eqref{eq: poisson's eq div} arises as the necessary condition for optimality for both $I[\cdot]$ and  $J[\cdot]$, with the structure of $J[\cdot]$ guaranteed by Theorem~\ref{thm: dirichlet principle}. For further discussion on Euler-Lagrange equations, see  \cite{dacorogna2024introduction}.

\section{Safety-Critical Control For High-Order Systems via Poisson Safety Functions}\label{appdx: higher-order cbfs}
We generalize the discussions in Section~\ref{sec: CBFs from Poisson} and formalize them for systems defined as integrator chains \eqref{eq: linear output dynamcis1} for outputs \eqref{eq: spatial coordinates} of relative degree $r \in \mathbb{N}$.
\subsection{CBF Backstepping for High-Order Systems}
Let $r\geq 2$ and suppose $h \in C^{2+k,\alpha}(\overline{\Oc})$ is a solution to \eqref{eq: poisson's eq appndx} satisfying Theorem~\ref{thm: main thm Safety Value Function}, for some $\alpha \in (0,1)$ and $k \in \mathbb{N}_0$ such that $r \leq k+1$. Define the following function via a backstepping procedure as denoted in  \cite{cohen2024constructive, taylor2022safe}:
\begin{align}\label{eq: CBF backstepping}
h_\mathrm{B}(\bvy) = h(\mathbf{y}) - \sum_{i=1}^{r-1}\frac{1}{2\mu_i}\Big\|\mathbf{y}^{(i)} - \mathbf{k}_i 
&\Big( \mathbf{y}, \cdots, \mathbf{y}^{(i-1)}, h(\mathbf{y}),\nonumber\\[-2ex]
& \ D h(\mathbf{y}), \cdots, D^{i} h(\mathbf{y})\Big)\Big\|^2
\end{align}
where $\bk_i \in C^{r-i+1}(\re^{3i} \times \re \times \re^{3} \times \cdots \times\re^{3^{i-1}}\times\re^{3^i}; \re^3)
$
for  $1\leq i \leq r-1$ are recursively designed auxiliary safe controllers (see \cite[Theorem 5]{taylor2022safe}). The time derivatives of $\bk_i$ are of the structure:
\begin{align}
\dot{\bk}_i(\cdot) &=\sum_{j = 1}^{i}\frac{\partial \bk_i}{\partial \by^{(j-1)}}(\cdot)\by^{(j)} \\
&= \Phi_i (\mathbf{y}, \cdots, \mathbf{y}^{(i)}, h(\mathbf{y}), D h(\mathbf{y}), \cdots, D^{i+1} h(\mathbf{y})).
\end{align}
The function \eqref{eq: CBF backstepping} defines a $0$-superlevel set given by:
\begin{align}\label{eq: high-order safe set}
    \Cc_\mathrm{B} = \left\{\bvy \in \re^{3r}  \, \big| \, h_\mathrm{B}(\bvy) \geq 0 \right\} \subset \Cc \times \re^{3(r - 1)},
\end{align}
we wish to render forward invariant.
To design a controller yielding unique solutions (\ie trajectories) to \eqref{eq: linear output dynamcis1}, 
the gradients of \eqref{eq: CBF backstepping} must be at least Lipschitz continuous. Consequently, the $r$-th partial derivatives of $h$ must be Lipschitz continuous, while the higher order partial derivatives (\ie $r+1$ and beyond) can be Hölder continuous, as they are not directly used. This condition arises from the fact that Hölder continuity guarantees the existence but not uniqueness of solutions  to ODEs.
The condition $r \leq k+1$ in Lemma~\ref{lemma: Lipschitz gradients} ensures this requirement is satisfied, and enables the following theorem.

\begin{theorem}\label{thm: main thm}(Forward Invariance for High-Order Systems) Let $\Omega \subset \re^3$ be an open, bounded and connected set with piecewise smooth boundary $\partial \Oc$. 
Consider the system \eqref{eq: linear output dynamcis1} for outputs with  relative degree $r\geq 2$, and a safe set $\Cc$ defined as the $0$-super level set of a function $h: \overline \Oc \rightarrow \re$. Suppose that $h$ is the solution to \eqref{eq: poisson's eq appndx} with a forcing function $f \in C^{k,\alpha}(\Omega;\re_{< 0})$ for some $\alpha \in (0,1)$ and $k \in \mathbb{N}_0$, and consider the function \eqref{eq: CBF backstepping} defining the $0$-super level set $\Cc_\mathrm{B}$ in \eqref{eq: high-order safe set}. 
If $r \leq k +1$,
then  given $\bvy(0) = \by_0 \in \Cc_\mathrm{B}$, there exists a locally Lipchitz continuous control law $ \bw= \bk(\bvy)$  for  \eqref{eq: linear output dynamcis1} such that: 
\begin{align}\label{eq: hdot backstepping}
    \dot h_\mathrm{B}(\bvy) = Dh_\mathrm{B}(\bvy) \cdot  \big(\bA\bvy + \bB\bk(\bvy)\big) \geq -\gamma h_\mathrm{B}(\bvy)
\end{align}
for all $\bvy \in \Cc_\mathrm{B}, \gamma >0$. Thus $\Cc_\mathrm{B}$ is rendered forward invariant and \eqref{eq: CBF backstepping} is a CBF for \eqref{eq: linear output dynamcis1}. 
\end{theorem}

\begin{proof} As shown in Theorem~\ref{thm: main thm Safety Value Function} and Proposition~\ref{lemma: relative degree 1}, it can verified using Hopf's Lemma and the weak minimum principle that the solution $h \in C^{2+k, \alpha}(\overline \Oc; \re_{\geq 0})$ guaranteed by Theorem~\ref{thm:regularity} satisfies $D h(\by) \neq 0$ when $h(\by)=0$, \ie when $\by \in \partial \Cc$. 
Taking the time derivative of \eqref{eq: CBF backstepping}, we have:
\begin{align}
    \dot h_\mathrm{B}(\bvy, \bw) &= Dh_\mathrm{B}(\bvy) \cdot \big(\bA\bvy + \bB\bw\big)  \\
    &\coloneqq Dh(\by)\cdot \dot{\by} - 
    \sum_{i=1}^{r-2}\frac{1}{\mu_i}
     ( \by^{(i)} - \bk_i)^\top ( \by^{(i+1)} - \dot{\bk}_i)  \nonumber \\
    &\qquad \qquad- \frac{1}{\mu_{r-1}}(\by^{(r-1)} - \bk_i)^\top ( \bw- \dot{\bk}_{r-1}), \nonumber 
\end{align}
where the time derivative $\dot{\bk}_{r-1}$ is of the form: 
\begin{align}
     \dot{\bk}_{r-1}(\cdot) &= \sum_{j = 1}^{r-1}\frac{\partial \bk_{r-1}}{\partial \by^{(j-1)}}(\cdot)\by^{j} \\
     &= \Phi_{r-1}\Big({\bvy}, h(\mathbf{y}), D h(\mathbf{y}), \cdots, D^{r} h(\mathbf{y})\Big).
\end{align}
Leveraging Lemma~\ref{lemma: Lipschitz gradients}, it follows that $D^{r} h$ is Lipschitz continuous on $\Cc$ for all $r \leq k +1$. Therefore, it follows from \cite[Theorem 5]{taylor2022safe} 
that there exists a locally Lipschitz continuous control law $ \bw= \bk(\bvy)$ satisfying \eqref{eq: hdot backstepping}. 
Therefore, Nagumo's theorem holds and $\Cc_\mathrm{B}$ is rendered forward invariant. As a result, $h_\mathrm{B}$ serves as a CBF for \eqref{eq: linear output dynamcis1}. 
\end{proof}

\subsection{High-Order Control Barrier Functions}

An alternative approach to constructing CBFs for systems of high relative degree is High-Order CBFs (HOCBFs) \cite{xiao2021high}. HOCBFs  define a CBF candidate recursively by differentiating $h$ until the input appears. Consider the system \eqref{eq: linear output dynamcis1} for outputs with a relative degree $r \geq 2$, and let $h \in C^{2+k,\alpha}(\overline{\Oc})$ be a solution to \eqref{eq: poisson's eq appndx} for some $\alpha \in (0,1)$ and $k \in \mathbb{N}_0$ such that $r \leq k+1$. Define $h_i$ recursively as follows:
\begin{align}
    h_0 &\defeq h, \\
    h_i &\defeq \dot{h}_{i-1} + \gamma_i h_{i-1}, \quad \quad i = 1,\cdots, r-1
\end{align}
with $\gamma_i > 0$ where each $h_i$ and $\dot{h}_i$ explicitly depends on:
\begin{align}
    h_i(\bty) &= h_i(\by, \cdots, \by^{(i)}, h(\by), Dh(\by), \cdots, D^ih(\by))\\
    \dot{h}_{i}(\bty)&= \sum_{j = 1}^{i+1}\frac{\partial h_i}{\partial \by^{(j-1)}}(\bty)\by^{(j)} \\
    &= \dot{h}_i(\by, \cdots, \by^{(i+1)}, h(\by), Dh(\by), ..., D^{i+1}h(\by)) \label{eq: HOCBF hdot explicit}
\end{align}
where $\bty = (\by, \cdots, \by^{(i)})$ represents the state variables up to the $i$-th derivative.
Each function $h_i$ defines a 0-superlevel set:
\[
\Cc_i = \left\{\bty \in \re^{3{i+1}}  \, \big| \,h_i(\bty) \geq 0 \right\} \quad i = 1,\cdots, r-1
\]
with $\Cc_0 = \Cc$. These sets are used to define a safe set as the intersection:
\begin{align}
    \Cc_\mathrm{H} \coloneqq \bigcap_{i=0}^{r-1} \Cc_i \subset \Cc,
\end{align}
which we aim to render forward invariant. The function:
\begin{align}\label{eq: candidate HOCBF}
    h_\mathrm{H}(\bvy) &\defeq h_{r-1}(\bvy) \\
    &= h_{r-1}(\bvy, h(\by), Dh(\by), \cdots, D^{r-1}h(\by)) \nonumber
\end{align}
is a HOCBF for \eqref{eq: linear output dynamcis1} restricted to $\Cc_\mathrm{H}$ if $\Cc_\mathrm{H}$ can be rendered forward invariant via a locally Lipschitz continuous controller $\bw= \bk(\bvy)$ satisfying:
\begin{align}\label{eq: hdot HOCBF}
    \dot h_\mathrm{H}(\bvy) = Dh_\mathrm{H}(\bvy) \cdot \big(\bA\bvy + \bB\bk(\bvy)\big) \geq -\gamma h_\mathrm{H}(\bvy)
\end{align}
for all $\bvy \in \Cc_\mathrm{H}$ with $\gamma >0$.
From \eqref{eq: HOCBF hdot explicit}, we observe that $\dot{h}_\mathrm{H}$ depends on $D^rh$. Therefore, if $h_\mathrm{H}$ is an HOCBF, then by Lemma~\ref{lemma: Lipschitz gradients}, there exists a locally Lipschitz continuous controller satisfying \eqref{eq: hdot HOCBF} if  $r \leq k+1$. However, to guarantee the existence of such a controller, HOCBFs further assume a uniform relative degree by requiring $Dh_\mathrm{H}(\bvy)\bB \neq 0$ for all $\bvy \in \Cc_\mathrm{H}$. However, one can verify that: 
\begin{align}
    Dh_\mathrm{H}(\bvy)\bB = Dh(\by),
\end{align}
and because the set $\Cc$ generated by \eqref{eq: poisson's eq appndx} is compact, the condition $Dh_\mathrm{H}(\bvy)\bB \neq 0$ is not guaranteed to hold for all $\bvy \in \Cc_\mathrm{H} \subset \Cc$. This leads to a notion of \textit{weak} relative degree \cite{XiaoWeakTAC} and requires further technical development to guarantee the forward invariance of $\Cc_\mathrm{H}$ as addressed in \cite{ong2024rectified}. Nevertheless, $h \in C^{2+k, \alpha}(\overline{\Oc})$ can be used to construct a \textit{candidate} HOCBF, $h_\mathrm{H}$ as defined in \eqref{eq: candidate HOCBF}.